\documentclass[11pt,twoside]{article}

\usepackage[utf8]{inputenc}
\usepackage[T1]{fontenc}

\usepackage{graphicx}
\usepackage{amsmath,amssymb,amsthm,bm,mathtools}
\usepackage{hyperref}
\usepackage{natbib}
\usepackage{fancyhdr}
\usepackage{geometry}
\usepackage{titlesec}
\usepackage{lipsum}
\usepackage{orcidlink}
\usepackage{booktabs}
\usepackage{placeins}
\usepackage{float}
\usepackage[ruled,vlined,linesnumbered]{algorithm2e}

\setcitestyle{authoryear}
\newtheorem{theorem}{Theorem}[section]
\newtheorem{proposition}[theorem]{Proposition}
\newtheorem{lemma}[theorem]{Lemma}
\newtheorem{corollary}[theorem]{Corollary}
\newtheorem{remark}[theorem]{Remark}

\hypersetup{
colorlinks=true,
linkcolor=blue,
filecolor=blue,
citecolor=blue,
urlcolor=black
}
\fancypagestyle{plain}{
\fancyhf{}

}
\title{
\vspace{-1.5em}
\hrule height 1.5pt
\vspace{0.8em}
PDGMM-VAE: A Variational Autoencoder with Adaptive Per-Dimension Gaussian Mixture Model Priors for Nonlinear ICA
\vspace{0.8em}
\hrule height 1.5pt
\vspace{1em}
}

\author{%
\begin{minipage}[t]{.48\textwidth}\centering\small
  \textbf{Yuan-Hao Wei\orcidlink{0000-0001-9439-0780}}\\
  Hong Kong Polytechnic University \\
  \texttt{Yuan-Hao.Wei@outlook.com}
\end{minipage}\hfill
\begin{minipage}[t]{.48\textwidth}\centering\small
  \textbf{Yan-Jie Sun\orcidlink{0000-0002-7967-6382}}\\
  Hong Kong Polytechnic University \\
  \texttt{Yanjie.Sun@connect.polyu.hk}
\end{minipage}
}

\date{}
\begin{document}
\maketitle
    \thispagestyle{plain}
        \begin{abstract}
            Independent component analysis is a core framework within blind source separation for recovering latent source signals from observed mixtures under statistical independence assumptions. In this work, we propose PDGMM-VAE, a source-oriented variational autoencoder in which each latent dimension, interpreted explicitly as an individual source component, is assigned its own adaptive Gaussian mixture model prior. The proposed framework imposes heterogeneous per-dimension prior constraints, enabling different latent dimensions to model different non-Gaussian source marginals within a unified probabilistic encoder-decoder architecture. The parameters of these source-specific GMM priors are not fixed in advance, but are jointly learned together with the encoder and decoder under the overall training objective. Beyond the model construction itself, we provide a theoretical analysis clarifying why adaptive per-dimension prior design is meaningful in this setting. In particular, we show that heterogeneous per-dimension priors reduce latent permutation symmetry relative to homogeneous shared priors, and we further show that the KL regularization induced by the adaptive GMM prior creates source-specific attraction behavior that helps explain source-wise specialization during training. We also clarify the relation of the proposed model to the standard VAE and provide a weak recovery statement in an idealized linear low-noise regime. Experimental results on both linear and nonlinear mixing problems show that PDGMM-VAE can recover latent source signals and fit source-specific non-Gaussian marginals effectively. These results suggest that adaptive per-dimension mixture-prior design provides a principled and promising direction for VAE-based ICA and source-oriented generative modeling.

        \end{abstract}

    
\section{Introduction}

        Independent component analysis (ICA) is a foundational framework for recovering latent source signals from observed mixtures under statistical independence assumptions. Within the broader field of blind source separation (BSS), ICA has been one of the most influential formulations, particularly in the linear case where the latent sources are typically assumed to be non-Gaussian \citep{hyvarinen2000ica,hyvarinen2001ica}. In many classical applications, such as speech mixtures, biomedical recordings, and multichannel sensor measurements, ICA has served as a principled route to source recovery and interpretable latent decomposition \citep{hyvarinen2000ica,hyvarinen2001ica}.
        
        Compared with linear ICA, nonlinear ICA is substantially more challenging. Classical results showed that, without additional assumptions, nonlinear mixtures generally do not admit the same kind of straightforward identifiability enjoyed in the linear case \citep{hyvarinen1999nonlinearica}. More recent work has therefore investigated what kinds of structure are sufficient to recover latent independent components. Important directions include the use of auxiliary variables and generalized contrastive learning \citep{hyvarinen2019auxica}, identifiable latent-variable formulations based on exponential-family structure and variational autoencoders \citep{khemakhem2020vaeica}, and more recent studies on unconditional or weakly restricted identifiability in nonlinear ICA \citep{willetts2021idontneedu,buchholz2022functionclasses,zheng2022sparsity}. These advances have also clarified the close relationship between nonlinear ICA, disentangled representation learning, and identifiable deep generative modeling \citep{hyvarinen2023review,kivva2022mixture}.
        
        Variational autoencoders (VAEs) provide a natural probabilistic framework for addressing ICA-like problems. When the latent variables are interpreted as source signals, the encoder may be viewed as a demixing mapping from observations to latent sources, while the decoder acts as a remixing or generative mapping back to the observation space. This perspective is especially appealing in nonlinear settings, where deep neural networks offer expressive modeling of complex mixing and demixing transformations. In particular, \citet{khemakhem2020vaeica} established a notable connection between VAEs and nonlinear ICA by showing that identifiable latent-variable models can be obtained under suitable structural assumptions. This connection suggests that VAE-based source models are not merely heuristic architectures, but part of a broader probabilistic route toward interpretable nonlinear source separation.
        
        At the same time, most existing VAE work using Gaussian mixture priors has largely pursued a different objective. A prominent line of research employs Gaussian-mixture latent structure for unsupervised clustering, deep clustering, or disentangled clustering, where the mixture components represent clusters, classes, or semantic facets of the data rather than source signals to be separated. Representative examples include GMVAE-style clustering models \citep{dilokthanakul2016gmvae}, Variational Deep Embedding (VaDE) \citep{jiang2017vade}, latent-tree or multidimensional clustering VAEs \citep{li2018ltvae}, Gaussian Mixture Variational Ladder Autoencoders \citep{willetts2019vlac}, and Multi-Facet Clustering Variational Autoencoders (MFCVAE) \citep{falck2021mfcvae}. These works have demonstrated the value of mixture priors for clustering-oriented latent organization, but they do not systematically study the setting in which each latent dimension is explicitly interpreted as an independent source and assigned its own source-specific Gaussian mixture prior for ICA.
        
        A related point is that, in our earlier Half-VAE study, we briefly noted the feasibility of assigning different Gaussian mixture priors to different latent dimensions in an encoder-free ICA setting \citep{wei2024halfvae}. However, that work focused on bypassing explicit inverse mapping by removing the encoder, rather than systematically developing a full encoder-decoder VAE for ICA. The present study takes the next step by formulating a source-oriented VAE in which each latent dimension is treated explicitly as an individual source and endowed with its own Gaussian mixture model prior. This per-dimension prior design allows different latent sources to follow distinct non-Gaussian distributions, thereby imposing heterogeneous independence-promoting constraints across latent dimensions.
        
        To this end, we propose \textbf{PDGMM-VAE}, a per-dimension Gaussian-mixture-model-prior variational autoencoder for nonlinear ICA. In the proposed framework, the encoder infers latent source variables from observed mixtures, the decoder reconstructs the mixtures from the inferred sources, and each latent source is regularized by its own learnable GMM prior. Importantly, the mixture weights, component means, and variances of these priors are not fixed in advance, but are adaptively learned and jointly optimized with the encoder and decoder parameters under the overall objective, thereby being automatically refined toward convergence during training. In contrast to clustering-oriented mixture-prior VAEs, our goal is not to partition data samples into latent classes, but to recover independent latent source signals under a probabilistic generative model. Experiments on both linear and nonlinear mixing problems show that the proposed method can effectively recover latent sources and achieve strong separation performance. More broadly, this work provides a systematic study of per-dimension GMM priors in VAE-based ICA and establishes a foundation for future research on structured priors, interpretability, and identifiability in generative source separation models.
        
\section{Related Work}

\subsection{ICA, BSS, and nonlinear identifiability}

ICA is one of the foundational frameworks for blind source separation, particularly when the latent sources are assumed to be statistically independent and non-Gaussian \citep{hyvarinen2000ica,hyvarinen2001ica}. In classical linear settings, this principle has enabled effective recovery of latent components from observed mixtures and has supported a wide range of applications in signal processing and machine learning \citep{hyvarinen2000ica}. However, extending ICA to nonlinear mixtures is considerably more difficult. Early work showed that nonlinear ICA is generally non-identifiable without additional assumptions \citep{hyvarinen1999nonlinearica}. Recent progress has therefore focused on identifying sufficient structure for recoverability, including auxiliary-variable formulations \citep{hyvarinen2019auxica}, identifiable deep latent-variable models \citep{khemakhem2020vaeica}, unconditional identifiability under restricted function classes \citep{buchholz2022functionclasses}, and broader analyses of sparsity or structural assumptions in nonlinear ICA \citep{zheng2022sparsity}. These developments have made nonlinear ICA an increasingly important bridge between source separation, disentanglement, and identifiable representation learning \citep{hyvarinen2023review}.

\subsection{VAE-based nonlinear ICA and generative source models}

The connection between VAEs and nonlinear ICA has become clearer in recent years. In particular, \citet{khemakhem2020vaeica} showed that a large family of identifiable nonlinear ICA models can be expressed within a VAE-like latent-variable framework under suitable assumptions. This result is important because it places probabilistic encoder-decoder models within the theoretical landscape of source recovery rather than treating them purely as black-box generative models. Related studies have further explored identifiability in deep generative models under mixture priors or without auxiliary information, highlighting both the promise and the subtle limitations of probabilistic latent-variable approaches for representation recovery \citep{willetts2021idontneedu,kivva2022mixture}.

\subsection{Gaussian-mixture-prior VAEs for clustering and disentanglement}

A separate and influential line of work uses Gaussian mixture priors in VAEs to induce clustering structure in latent space. GMVAE-type models explicitly adopt Gaussian mixture latent distributions for deep unsupervised clustering \citep{dilokthanakul2016gmvae}, while VaDE combines a GMM with a VAE to obtain a generative clustering model with variational inference over latent embeddings and cluster assignments \citep{jiang2017vade}. Subsequent works extended this idea in different directions. LTVAE introduced latent superstructures to support multidimensional clustering \citep{li2018ltvae}; Gaussian Mixture Variational Ladder Autoencoders organized clustering variables hierarchically across latent layers \citep{willetts2019vlac}; and MFCVAE further developed this line into a multi-facet clustering framework in which each facet is associated with a mixture-of-Gaussians prior and the objective is to discover multiple meaningful partitions of the data \citep{falck2021mfcvae}. Although these models differ in architecture and training strategy, they share a common emphasis: the mixture prior is used primarily to organize data samples into clusters or semantic facets.

This clustering-oriented use of Gaussian mixtures is fundamentally different from the goal of the present work. In our setting, the latent dimensions are not intended to encode cluster identity or facet-specific sample groupings. Instead, each latent dimension is explicitly interpreted as an individual source signal, and the role of the prior is to model source-specific non-Gaussianity so as to promote source separation. Thus, while previous mixture-prior VAEs are highly relevant as methodological neighbors, they do not directly address the problem of per-dimension GMM priors for ICA.

\subsection{Relation to our previous work}

The present study is also closely related to our previous line of research on VAE-based ICA and source separation, in which latent dimensions are explicitly interpreted as source signals and different forms of structured priors are introduced to improve source recovery, interpretability, and adaptability to different data characteristics \citep{wei2024halfvae,wei2024innovative,wei2025structured,wei2026ar}. From this broader perspective, the current work should not be viewed as an isolated model, but rather as part of an ongoing effort to develop source-oriented VAE frameworks in which prior design plays a central role in guiding the separation of latent components.

Among our earlier studies, Half-VAE \citep{wei2024halfvae} investigated an encoder-free formulation for ICA, where latent variables were directly optimized as trainable parameters in order to bypass explicit inverse mapping. That work briefly suggested that different latent dimensions could be equipped with different GMM priors, indicating the feasibility of per-dimension mixture-based priors for ICA, but it did not provide a systematic study of a full encoder-decoder VAE with source-specific GMM prior learning. More broadly, our previous works have explored several complementary directions for structured prior design. PAVAE \citep{wei2024innovative} further examined adaptive prior-oriented VAE formulations for source inference, emphasizing the role of prior parameterization in improving latent decomposition. Structured Kernel Regression VAE (SKR-VAE) \citep{wei2025structured} focused on kernel-structured priors as a computationally efficient surrogate for GP-based latent modeling, making such structure more practical for ICA settings in which different latent dimensions may exhibit different correlation patterns. AR-Flow VAE \citep{wei2026ar}, in turn, extended this line by introducing autoregressive flow priors to provide more flexible modeling of complex non-Gaussian behavior and structured latent dependencies.

Taken together, these studies share a common methodological philosophy: different source models may require different prior structures, and structured prior design is a key mechanism for adapting VAE-based ICA/BSS models to different signal assumptions and data types. Under this view, the proposed PDGMM-VAE occupies a distinct and well-motivated position in our broader research program. Unlike kernel-based or flow-based priors, which are especially suitable when one wishes to capture temporal, spatial, or more complex dependency structures, the per-dimension GMM priors adopted here are particularly appropriate for the present ICA-oriented setting, where the simulated sources are i.i.d.\ and the main emphasis is on source-wise non-Gaussianity and mutual independence. Thus, the contribution of this work is not merely to reuse a previously mentioned idea, but to develop the per-dimension GMM-prior direction into a full VAE framework and to study it systematically for linear and nonlinear ICA.

\section{Methodology}

\subsection{Problem formulation}

Consider an observed mixture sequence
\begin{equation}
\mathbf{Y}
=
[\mathbf{y}_1,\mathbf{y}_2,\dots,\mathbf{y}_T]^\top
\in \mathbb{R}^{T\times m},
\end{equation}
where $T$ denotes the number of samples and $m$ denotes the observation dimension. The objective is to infer a latent source sequence
\begin{equation}
\mathbf{Z}
=
[\mathbf{z}_1,\mathbf{z}_2,\dots,\mathbf{z}_T]^\top
\in \mathbb{R}^{T\times n},
\end{equation}
where $n$ is the number of latent source components.

In the proposed framework, each latent dimension is interpreted explicitly as a source-related component. Under this view, the source separation problem is reformulated as a latent-variable generative modeling problem in which the latent variables correspond to source-like components and the observations are treated as mixtures generated from these latent variables through an unknown mapping. This motivates the use of a variational autoencoder (VAE), where the encoder plays the role of an observation-to-latent demixing map and the decoder plays the role of a latent-to-observation remixing map.

For each sample index $t=1,\dots,T$, the observation $\mathbf{y}_t$ is associated with a latent vector $\mathbf{z}_t\in\mathbb{R}^{n}$. The central modeling question is then how to impose a latent prior that is sufficiently expressive to promote source recovery. Instead of assigning the same simple prior to all latent dimensions, we endow each latent dimension with its own adaptive one-dimensional Gaussian mixture model (GMM) prior. This per-dimension prior design allows different latent dimensions to model different non-Gaussian source marginals and thereby promotes source-wise specialization.

\subsection{Generative model and variational posterior}

We write the joint distribution as
\begin{equation}
p_{\theta,\psi}(\mathbf{Y},\mathbf{Z})
=
p_{\theta}(\mathbf{Y}\mid\mathbf{Z})\,p_{\psi}(\mathbf{Z}),
\end{equation}
where $\theta$ denotes the decoder parameters and $\psi$ collects the parameters of all per-dimension GMM priors.

Assuming factorization over samples, we have
\begin{equation}
p_{\theta}(\mathbf{Y}\mid\mathbf{Z})
=
\prod_{t=1}^{T} p_{\theta}(\mathbf{y}_t\mid \mathbf{z}_t),
\qquad
p_{\psi}(\mathbf{Z})
=
\prod_{t=1}^{T} p_{\psi}(\mathbf{z}_t).
\end{equation}
Moreover, because the prior is factorized across latent dimensions,
\begin{equation}
p_{\psi}(\mathbf{z}_t)
=
\prod_{j=1}^{n} p_{\psi_j}(z_{t,j}).
\end{equation}

The decoder $g_\theta(\cdot)$ maps a latent vector back to the observation space:
\begin{equation}
\hat{\mathbf{y}}_t = g_\theta(\mathbf{z}_t),
\qquad
\hat{\mathbf{y}}_t\in\mathbb{R}^{m}.
\end{equation}
When no hidden layer is used, $g_\theta(\cdot)$ reduces to a linear mapping; when nonlinear hidden layers are included, the model can represent more general nonlinear mixing mechanisms.

Since the exact posterior $p_{\theta,\psi}(\mathbf{Z}\mid\mathbf{Y})$ is intractable in general, we introduce a variational posterior
\begin{equation}
q_\phi(\mathbf{Z}\mid\mathbf{Y})
=
\prod_{t=1}^{T} q_\phi(\mathbf{z}_t\mid\mathbf{y}_t),
\end{equation}
parameterized by an encoder network $f_\phi(\cdot)$. For each sample $t$,
\begin{equation}
\boldsymbol{\mu}_t = f_\phi(\mathbf{y}_t)
=
[\mu_{t,1},\mu_{t,2},\dots,\mu_{t,n}]^\top.
\end{equation}
The approximate posterior is chosen to be factorized across latent dimensions:
\begin{equation}
q_\phi(\mathbf{z}_t\mid\mathbf{y}_t)
=
\prod_{j=1}^{n} q_\phi(z_{t,j}\mid\mathbf{y}_t),
\end{equation}
with
\begin{equation}
q_\phi(z_{t,j}\mid\mathbf{y}_t)
=
\mathcal{N}(z_{t,j}\mid \mu_{t,j}, \sigma_j^2),
\end{equation}
where $\mu_{t,j}$ depends on $\mathbf{y}_t$ through the encoder, while $\sigma_j^2$ is a learnable scalar variance shared by the $j$-th latent source dimension across all sample indices. Therefore,
\begin{equation}
q_\phi(\mathbf{Z}\mid\mathbf{Y})
=
\prod_{t=1}^{T}\prod_{j=1}^{n}
\mathcal{N}(z_{t,j}\mid \mu_{t,j}, \sigma_j^2).
\end{equation}

This parameterization keeps the posterior stochastic while avoiding an excessive number of free variance parameters. It also preserves the source-wise interpretation by assigning one global posterior variance to each latent dimension.

\subsection{Reparameterization of latent variables}

To enable gradient-based optimization through stochastic latent variables, we use the reparameterization trick. For each $t$ and $j$,
\begin{equation}
z_{t,j}
=
\mu_{t,j} + \sigma_j \epsilon_{t,j},
\qquad
\epsilon_{t,j}\sim \mathcal{N}(0,1).
\end{equation}
Equivalently, in vector form,
\begin{equation}
\mathbf{z}_t
=
\boldsymbol{\mu}_t + \boldsymbol{\sigma}\odot \boldsymbol{\epsilon}_t,
\qquad
\boldsymbol{\epsilon}_t \sim \mathcal{N}(\mathbf{0},\mathbf{I}),
\end{equation}
where
\begin{equation}
\boldsymbol{\sigma}
=
[\sigma_1,\sigma_2,\dots,\sigma_n]^\top
\end{equation}
and $\odot$ denotes element-wise multiplication.

\subsection{Per-dimension Gaussian mixture prior}

Instead of imposing the same standard normal prior on every latent dimension, we assign to each latent source dimension $j$ an independent one-dimensional GMM prior:
\begin{equation}
p_{\psi_j}(z_{t,j})
=
\sum_{k=1}^{K}
\pi_{j,k}
\,
\mathcal{N}\!\left(
z_{t,j}\mid \mu_{j,k}^{(p)}, (\sigma_{j,k}^{(p)})^2
\right),
\end{equation}
where $K$ is the number of mixture components, $\pi_{j,k}$ is the mixture weight, $\mu_{j,k}^{(p)}$ is the component mean, and $(\sigma_{j,k}^{(p)})^2$ is the component variance for the $k$-th component of latent dimension $j$.

The mixture weights satisfy
\begin{equation}
\pi_{j,k}\ge 0,
\qquad
\sum_{k=1}^{K}\pi_{j,k}=1,
\end{equation}
and are parameterized through softmax logits $\alpha_{j,k}$:
\begin{equation}
\pi_{j,k}
=
\frac{\exp(\alpha_{j,k})}{\sum_{\ell=1}^{K}\exp(\alpha_{j,\ell})}.
\end{equation}
Similarly, the component variances are parameterized in log-form to guarantee positivity:
\begin{equation}
(\sigma_{j,k}^{(p)})^2 = \exp(\eta_{j,k}).
\end{equation}

Because the prior is factorized across source dimensions,
\begin{equation}
p_{\psi}(\mathbf{z}_t)
=
\prod_{j=1}^{n} p_{\psi_j}(z_{t,j})
=
\prod_{j=1}^{n}
\left[
\sum_{k=1}^{K}
\pi_{j,k}
\,
\mathcal{N}\!\left(
z_{t,j}\mid \mu_{j,k}^{(p)}, (\sigma_{j,k}^{(p)})^2
\right)
\right].
\end{equation}
Hence the prior over the whole latent dataset is
\begin{equation}
p_{\psi}(\mathbf{Z})
=
\prod_{t=1}^{T} p_{\psi}(\mathbf{z}_t)
=
\prod_{t=1}^{T}\prod_{j=1}^{n}
p_{\psi_j}(z_{t,j}).
\end{equation}

Correspondingly, the log-prior density can be written as
\begin{equation}
\log p_{\psi}(\mathbf{Z})
=
\sum_{t=1}^{T}\sum_{j=1}^{n}
\log p_{\psi_j}(z_{t,j}),
\end{equation}
with
\begin{equation}
\log p_{\psi_j}(z_{t,j})
=
\log
\left[
\sum_{k=1}^{K}
\pi_{j,k}
\,
\mathcal{N}\!\left(
z_{t,j}\mid \mu_{j,k}^{(p)}, (\sigma_{j,k}^{(p)})^2
\right)
\right].
\end{equation}

This prior is substantially more flexible than the conventional isotropic Gaussian prior. In particular, it allows each latent dimension to model multimodal, heavy-tailed, asymmetric, or otherwise non-Gaussian source marginals. More importantly, because different latent dimensions are endowed with different prior parameters, the resulting latent regularization is heterogeneous across dimensions rather than homogeneous.
\begin{remark}[One-dimensional marginal expressiveness of Gaussian mixtures]
The role of the per-dimension GMM prior is not only algorithmic but also functional. Let $f$ be a continuous probability density on $\mathbb R$ such that $f(x)\to 0$ as $|x|\to\infty$. Then for any $\varepsilon>0$, there exists a finite Gaussian mixture
\begin{equation}
g_K(x)
=
\sum_{k=1}^{K}\pi_k\,\mathcal{N}(x\mid m_k,v_k),
\qquad
\pi_k>0,\quad
\sum_{k=1}^{K}\pi_k=1,\quad
v_k>0,
\end{equation}
such that
\begin{equation}
\|f-g_K\|_{L_1(\mathbb R)}<\varepsilon.
\end{equation}
Under additional regularity and tail conditions, one may further obtain arbitrarily small divergence in stronger senses such as KL divergence. Therefore, from the viewpoint of one-dimensional source marginals, the adopted per-dimension GMM prior family is sufficiently expressive to approximate a broad class of continuous non-Gaussian source distributions.

This observation should be interpreted carefully. It does not by itself imply exact source recovery or full identifiability, since recoverability still depends on the decoder family, the variational posterior family, the optimization path, and any additional structural assumptions. It does, however, show that the prior family itself is not the main bottleneck at the level of one-dimensional marginal expressiveness.
\end{remark}

\subsection{Observation model, ELBO interpretation, and implemented training objective}

For probabilistic interpretation, one may introduce a Gaussian observation model
\begin{equation}
p_{\theta}(\mathbf{y}_t\mid\mathbf{z}_t)
=
\mathcal{N}\!\left(
\mathbf{y}_t \mid g_{\theta}(\mathbf{z}_t), \sigma_y^2 \mathbf{I}
\right).
\end{equation}
Under this model,
\begin{equation}
\log p_{\theta}(\mathbf{y}_t\mid\mathbf{z}_t)
=
-\frac{m}{2}\log(2\pi\sigma_y^2)
-
\frac{1}{2\sigma_y^2}
\|\mathbf{y}_t-g_\theta(\mathbf{z}_t)\|_2^2.
\end{equation}
Hence a mean-squared reconstruction penalty can be viewed as a Gaussian negative log-likelihood up to additive and multiplicative constants.

The marginal log-likelihood of the observations is
\begin{equation}
\log p_{\theta,\psi}(\mathbf{Y})
=
\log \int p_{\theta,\psi}(\mathbf{Y},\mathbf{Z})\,d\mathbf{Z}.
\end{equation}
Introducing the variational posterior $q_\phi(\mathbf{Z}\mid\mathbf{Y})$ and applying Jensen's inequality yields the evidence lower bound (ELBO)
\begin{equation}
\mathcal{J}_{\mathrm{ELBO}}
=
\mathbb{E}_{q_\phi(\mathbf{Z}\mid\mathbf{Y})}
\left[
\log p_{\theta}(\mathbf{Y}\mid\mathbf{Z})
+
\log p_{\psi}(\mathbf{Z})
-
\log q_\phi(\mathbf{Z}\mid\mathbf{Y})
\right].
\end{equation}
Using the factorization over samples and latent dimensions, we obtain
\begin{equation}
\mathcal{J}_{\mathrm{ELBO}}
=
\sum_{t=1}^{T}
\mathbb{E}_{q_\phi(\mathbf{z}_t\mid\mathbf{y}_t)}
\left[
\log p_\theta(\mathbf{y}_t\mid\mathbf{z}_t)
\right]
-
\sum_{t=1}^{T}\sum_{j=1}^{n}
\mathbb{E}_{q_\phi(z_{t,j}\mid\mathbf{y}_t)}
\left[
\log q_\phi(z_{t,j}\mid\mathbf{y}_t)
-
\log p_{\psi_j}(z_{t,j})
\right].
\end{equation}

For the adopted Gaussian posterior,
\begin{equation}
\log q_\phi(\mathbf{Z}\mid\mathbf{Y})
=
\sum_{t=1}^{T}\sum_{j=1}^{n}
\log \mathcal{N}(z_{t,j}\mid \mu_{t,j}, \sigma_j^2),
\end{equation}
and
\begin{equation}
\log \mathcal{N}(z_{t,j}\mid \mu_{t,j}, \sigma_j^2)
=
-\frac{1}{2}\log(2\pi)
-\frac{1}{2}\log \sigma_j^2
-\frac{(z_{t,j}-\mu_{t,j})^2}{2\sigma_j^2}.
\end{equation}

The prior term
\begin{equation}
\mathbb{E}_{q_\phi(z_{t,j}\mid\mathbf{y}_t)}
\left[
\log p_{\psi_j}(z_{t,j})
\right]
=
\mathbb{E}_{q_\phi(z_{t,j}\mid\mathbf{y}_t)}
\left[
\log
\sum_{k=1}^{K}
\pi_{j,k}
\mathcal{N}\!\left(
z_{t,j}\mid \mu_{j,k}^{(p)}, (\sigma_{j,k}^{(p)})^2
\right)
\right]
\end{equation}
does not admit a simple closed form in general because of the logarithm of the mixture sum.

In implementation, the loss is evaluated by using one reparameterized latent sample
\begin{equation}
\tilde z_{t,j}
=
\mu_{t,j}+\sigma_j\epsilon_{t,j},
\qquad
\epsilon_{t,j}\sim\mathcal N(0,1),
\end{equation}
or equivalently $\widetilde{\mathbf Z}=[\tilde{\mathbf z}_1,\dots,\tilde{\mathbf z}_T]^\top$.
The reconstruction term is the mean squared error
\begin{equation}
\mathcal{L}_{\mathrm{rec}}
=
\frac{1}{Tm}
\sum_{t=1}^{T}
\|\hat{\mathbf y}_t-\mathbf y_t\|_2^2,
\qquad
\hat{\mathbf y}_t=g_\theta(\tilde{\mathbf z}_t),
\end{equation}
and the latent regularization term is computed as
\begin{equation}
\mathcal{L}_{\mathrm{KL}}
=
\beta\cdot
\frac{
\log q_\phi(\widetilde{\mathbf Z}\mid\mathbf Y)
-
\log p_\psi(\widetilde{\mathbf Z})
}{Tm}.
\end{equation}
Therefore, the implemented training objective is
\begin{equation}
\mathcal{L}
=
\mathcal{L}_{\mathrm{rec}}+\mathcal{L}_{\mathrm{KL}}
=
\frac{1}{Tm}
\sum_{t=1}^{T}
\|\hat{\mathbf y}_t-\mathbf y_t\|_2^2
+
\beta\cdot
\frac{
\log q_\phi(\widetilde{\mathbf Z}\mid\mathbf Y)
-
\log p_\psi(\widetilde{\mathbf Z})
}{Tm}.
\end{equation}

Expanding the two log-density terms gives
\begin{equation}
\log q_\phi(\widetilde{\mathbf Z}\mid\mathbf Y)
=
\sum_{t=1}^{T}\sum_{j=1}^{n}
\log \mathcal N(\tilde z_{t,j}\mid \mu_{t,j},\sigma_j^2),
\end{equation}
and
\begin{equation}
\log p_\psi(\widetilde{\mathbf Z})
=
\sum_{t=1}^{T}\sum_{j=1}^{n}
\log p_{\psi_j}(\tilde z_{t,j}).
\end{equation}
Hence the implemented objective can be written explicitly as
\begin{equation}
\mathcal{L}
=
\frac{1}{Tm}
\sum_{t=1}^{T}
\|\hat{\mathbf y}_t-\mathbf y_t\|_2^2
+
\beta\cdot
\frac{
\sum_{t=1}^{T}\sum_{j=1}^{n}
\log \mathcal N(\tilde z_{t,j}\mid \mu_{t,j},\sigma_j^2)
-
\sum_{t=1}^{T}\sum_{j=1}^{n}
\log p_{\psi_j}(\tilde z_{t,j})
}{Tm}.
\end{equation}

The ELBO above provides the variational interpretation of the model. In practice, the training loss is implemented as a normalized single-sample objective with a weighting coefficient $\beta$ that balances reconstruction fidelity and latent regularization.

\subsection{Symmetry reduction under homogeneous and heterogeneous priors}

A key theoretical point of the proposed design is that a per-dimension heterogeneous prior does more than increase flexibility: it also weakens latent-space symmetries that are present under homogeneous priors. The following analysis uses the assumption that the encoder and decoder model classes are closed under latent coordinate permutations. This is a natural idealized assumption in the present context: if the encoder and decoder are implemented by sufficiently expressive multilayer perceptrons, then composing them with a permutation matrix in latent space does not change the representational nature of the model class, but only corresponds to a relabeling of latent coordinates. Accordingly, the assumption is not meant as a restrictive architectural constraint, but as a formal way to express that the network class is rich enough to absorb latent permutations.

\begin{theorem}[Permutation symmetry under homogeneous priors]
Suppose the latent prior takes the homogeneous factorized form
\begin{equation}
p_0(\mathbf{z}_t)
=
\prod_{j=1}^{n}\rho(z_{t,j}),
\end{equation}
where all latent dimensions share the same one-dimensional density $\rho$. Let $\Pi$ be any $n\times n$ permutation matrix. Assume that the encoder and decoder model classes are closed under latent coordinate permutations, in the sense that for every encoder-decoder pair $(f_\phi,g_\theta)$ there exists a transformed pair $(f_{\phi^\Pi},g_{\theta^\Pi})$ satisfying
\begin{equation}
f_{\phi^\Pi}(\mathbf{y})
=
\Pi f_\phi(\mathbf{y}),
\qquad
g_{\theta^\Pi}(\mathbf{z})
=
g_\theta(\Pi^\top \mathbf{z}).
\end{equation}
Then the ELBO is invariant under the induced latent-coordinate permutation.
\end{theorem}

\noindent\textit{Proof.}
Define the transformed posterior by
\begin{equation}
q_{\phi^\Pi}(\mathbf{z}\mid\mathbf{y})
=
q_\phi(\Pi^\top \mathbf{z}\mid\mathbf{y}).
\end{equation}
For the reconstruction term,
\begin{align}
\mathbb{E}_{q_{\phi^\Pi}(\mathbf{z}\mid\mathbf{y})}
\left[
\log p_{\theta^\Pi}(\mathbf{y}\mid\mathbf{z})
\right]
&=
\int
q_\phi(\Pi^\top\mathbf{z}\mid\mathbf{y})
\log p_\theta(\mathbf{y}\mid \Pi^\top\mathbf{z})
\,d\mathbf{z}.
\end{align}
Using the change of variable $\mathbf{u}=\Pi^\top\mathbf{z}$ and $|\det \Pi|=1$, this becomes
\begin{equation}
\mathbb{E}_{q_{\phi}(\mathbf{u}\mid\mathbf{y})}
\left[
\log p_{\theta}(\mathbf{y}\mid\mathbf{u})
\right].
\end{equation}

For the KL term,
\begin{align}
\mathrm{KL}\!\left(
q_{\phi^\Pi}(\mathbf{z}\mid\mathbf{y})\,\|\,p_0(\mathbf{z})
\right)
&=
\int
q_\phi(\Pi^\top\mathbf{z}\mid\mathbf{y})
\log
\frac{q_\phi(\Pi^\top\mathbf{z}\mid\mathbf{y})}{p_0(\mathbf{z})}
\,d\mathbf{z}.
\end{align}
Again setting $\mathbf{u}=\Pi^\top\mathbf{z}$ gives
\begin{equation}
\int
q_\phi(\mathbf{u}\mid\mathbf{y})
\log
\frac{q_\phi(\mathbf{u}\mid\mathbf{y})}{p_0(\Pi\mathbf{u})}
\,d\mathbf{u}.
\end{equation}
Under the homogeneous prior,
\begin{equation}
p_0(\Pi\mathbf{u})
=
\prod_{j=1}^{n}\rho((\Pi\mathbf{u})_j)
=
\prod_{j=1}^{n}\rho(u_j)
=
p_0(\mathbf{u}),
\end{equation}
so the KL term is unchanged. Therefore the ELBO is invariant under the permutation. \hfill$\square$

The theorem implies that when all latent dimensions share the same prior family, the latent labels themselves carry no intrinsic mathematical identity. Consequently, optimization may admit many equivalent solutions related by latent-coordinate permutations.

\begin{theorem}[Residual symmetry under heterogeneous per-dimension priors]
Let the prior be
\begin{equation}
p_{\psi}(\mathbf{z}_t)
=
\prod_{j=1}^{n} p_{\psi_j}(z_{t,j}),
\end{equation}
where $p_{\psi_j}$ denotes the GMM prior associated with the $j$-th latent dimension. Define the stabilizer subgroup
\begin{equation}
G_\psi
=
\left\{
\Pi\in S_n:
p_{\psi_j}=p_{\psi_{\Pi(j)}} \text{ a.e. for all } j
\right\}.
\end{equation}
Under the same encoder/decoder closure assumptions as in Theorem 1, the ELBO remains invariant only under permutations $\Pi\in G_\psi$. In particular, if the per-dimension prior densities are pairwise distinct, the only remaining permutation symmetry is the identity.
\end{theorem}

\noindent\textit{Proof.}
If $\Pi\in G_\psi$, then by definition
\begin{equation}
p_{\psi_j}=p_{\psi_{\Pi(j)}} \quad \text{a.e. for all } j,
\end{equation}
and therefore
\begin{equation}
p_\psi(\Pi\mathbf{z})
=
\prod_{j=1}^{n} p_{\psi_j}((\Pi\mathbf{z})_j)
=
\prod_{j=1}^{n} p_{\psi_{\Pi(j)}}(z_{\Pi(j)})
=
p_\psi(\mathbf{z}).
\end{equation}
Hence the same change-of-variable argument as in Theorem 1 shows that the ELBO is invariant.

Conversely, if the ELBO is invariant under $\Pi$, then in particular the prior contribution must be invariant, which requires
\begin{equation}
p_\psi(\Pi\mathbf{z})=p_\psi(\mathbf{z})
\quad \text{for almost every } \mathbf{z}.
\end{equation}
Fixing all coordinates except one shows that this is only possible if
\begin{equation}
p_{\psi_j}=p_{\psi_{\Pi(j)}} \quad \text{a.e. for all } j.
\end{equation}
Thus $\Pi\in G_\psi$. This proves the claim. \hfill$\square$

Theorem 2 formalizes a central point of PDGMM-VAE: heterogeneous priors reduce the permutation freedom that remains under homogeneous latent regularization. Therefore, source-wise specialization is not only an empirical phenomenon but also a consequence of reduced objective symmetry.

\subsection{KL-induced attraction and source-wise specialization}

The KL term does not merely penalize latent-prior mismatch abstractly. In the present model, it induces a concrete local geometry in latent space. Because each latent dimension is regularized by its own adaptive one-dimensional GMM prior, the resulting KL contribution creates source-specific attraction fields. This subsection makes that mechanism more explicit.

\begin{proposition}[Upper-envelope bound for the GMM KL term]
\label{prop:gmm_upper_envelope}
Fix one latent dimension $j$, and write
\begin{equation}
q(z)=\mathcal{N}(z\mid m,s^2),
\qquad
p_j(z)=\sum_{k=1}^{K}\pi_k\mathcal{N}(z\mid \mu_k,v_k).
\end{equation}
Then for every component $k$,
\begin{equation}
\mathrm{KL}(q\|p_j)
\le
\mathrm{KL}\!\left(
q \,\middle\|\, \mathcal{N}(\mu_k,v_k)
\right)
-
\log \pi_k.
\end{equation}
\end{proposition}

\noindent\textit{Proof.}
For every $z$ and every $k$,
\begin{equation}
p_j(z)
=
\sum_{\ell=1}^{K}\pi_\ell \mathcal{N}(z\mid \mu_\ell,v_\ell)
\ge
\pi_k \mathcal{N}(z\mid \mu_k,v_k).
\end{equation}
Taking logarithms gives
\begin{equation}
-\log p_j(z)
\le
-\log \pi_k
-
\log \mathcal{N}(z\mid \mu_k,v_k).
\end{equation}
Taking expectation with respect to $q$ and adding $\mathbb E_q[\log q(z)]$ on both sides yields
\begin{equation}
\mathrm{KL}(q\|p_j)
\le
\mathrm{KL}\!\left(
q \,\middle\|\, \mathcal{N}(\mu_k,v_k)
\right)
-
\log \pi_k.
\end{equation}
This proves the claim. \hfill$\square$

Proposition~\ref{prop:gmm_upper_envelope} shows that each mixture component induces a Gaussian envelope for the one-dimensional KL term. However, the specialization mechanism can be stated more sharply: in a dominant-component regime, the KL term is not only upper-bounded by a single-component surrogate, but is locally equal to such a surrogate up to an explicit correction term.

\begin{proposition}[Exact local decomposition under a designated dominant component]
\label{prop:gmm_exact_local_decomposition}
Fix one latent dimension $j$, and let
\begin{equation}
q(z)=\mathcal{N}(z\mid m,s^2),
\qquad
p_j(z)=\sum_{k=1}^{K}\pi_k\mathcal{N}(z\mid \mu_k,v_k).
\end{equation}
For any designated component $k^\star$, define
\begin{equation}
\rho_{k^\star}(z)
:=
\frac{
\sum_{\ell\neq k^\star}\pi_\ell \mathcal{N}(z\mid \mu_\ell,v_\ell)
}{
\pi_{k^\star}\mathcal{N}(z\mid \mu_{k^\star},v_{k^\star})
}.
\end{equation}
Then the KL term admits the exact identity
\begin{equation}
\mathrm{KL}(q\|p_j)
=
\mathrm{KL}\!\left(
q \,\middle\|\, \mathcal{N}(\mu_{k^\star},v_{k^\star})
\right)
-
\log \pi_{k^\star}
-
\mathbb E_q\!\left[\log\!\bigl(1+\rho_{k^\star}(z)\bigr)\right].
\end{equation}
Consequently,
\begin{equation}
0
\le
\mathrm{KL}\!\left(
q \,\middle\|\, \mathcal{N}(\mu_{k^\star},v_{k^\star})
\right)
-
\log \pi_{k^\star}
-
\mathrm{KL}(q\|p_j)
=
\mathbb E_q\!\left[\log\!\bigl(1+\rho_{k^\star}(z)\bigr)\right].
\end{equation}

Moreover, if there exists a measurable set $A\subset\mathbb R$ such that $q(A)\ge 1-\delta$ and
\begin{equation}
\rho_{k^\star}(z)\le \eta,
\qquad z\in A,
\end{equation}
and if
\begin{equation}
M_{A^c}
:=
\sup_{z\notin A}\log\!\bigl(1+\rho_{k^\star}(z)\bigr)
<
+\infty,
\end{equation}
then
\begin{equation}
0
\le
\mathrm{KL}\!\left(
q \,\middle\|\, \mathcal{N}(\mu_{k^\star},v_{k^\star})
\right)
-
\log \pi_{k^\star}
-
\mathrm{KL}(q\|p_j)
\le
(1-\delta)\log(1+\eta)+\delta M_{A^c}.
\end{equation}
\end{proposition}

\noindent\textit{Proof.}
By definition,
\begin{equation}
p_j(z)
=
\pi_{k^\star}\mathcal{N}(z\mid \mu_{k^\star},v_{k^\star})
\left[
1+\rho_{k^\star}(z)
\right].
\end{equation}
Hence
\begin{equation}
\log p_j(z)
=
\log \pi_{k^\star}
+
\log \mathcal{N}(z\mid \mu_{k^\star},v_{k^\star})
+
\log\!\bigl(1+\rho_{k^\star}(z)\bigr).
\end{equation}
Taking expectation with respect to $q$ and subtracting from $\mathbb E_q[\log q(z)]$ gives
\begin{align}
\mathrm{KL}(q\|p_j)
&=
\mathbb E_q[\log q(z)]
-
\mathbb E_q[\log p_j(z)] \\
&=
\mathrm{KL}\!\left(
q \,\middle\|\, \mathcal{N}(\mu_{k^\star},v_{k^\star})
\right)
-
\log \pi_{k^\star}
-
\mathbb E_q\!\left[\log\!\bigl(1+\rho_{k^\star}(z)\bigr)\right].
\end{align}
This proves the exact identity.

For the bound, split the expectation over $A$ and $A^c$:
\begin{equation}
\mathbb E_q\!\left[\log(1+\rho_{k^\star}(z))\right]
=
\mathbb E_q\!\left[\log(1+\rho_{k^\star}(z))\mathbf{1}_{A}\right]
+
\mathbb E_q\!\left[\log(1+\rho_{k^\star}(z))\mathbf{1}_{A^c}\right].
\end{equation}
On $A$, we have $\log(1+\rho_{k^\star}(z))\le \log(1+\eta)$; on $A^c$, it is bounded by $M_{A^c}$. Therefore
\begin{equation}
\mathbb E_q\!\left[\log(1+\rho_{k^\star}(z))\right]
\le
q(A)\log(1+\eta)+q(A^c)M_{A^c}
\le
(1-\delta)\log(1+\eta)+\delta M_{A^c}.
\end{equation}
This proves the claim. \hfill$\square$

Proposition~\ref{prop:gmm_exact_local_decomposition} shows that, when one mixture component dominates over the main support of the Gaussian variational posterior, the one-dimensional KL term is locally almost the same as matching a single Gaussian envelope. Therefore, the mode-seeking or basin-locking behavior of the KL term is not merely a heuristic statement: it is induced directly by the log-sum-exp structure of the Gaussian mixture prior.

The next result turns this local surrogate geometry into an explicit attraction statement for the posterior mean.

\begin{proposition}[Posterior-mean attraction under the reparameterized KL term]
\label{prop:posterior_mean_attraction}
Fix one latent dimension $j$, and write
\begin{equation}
q(z)=\mathcal{N}(z\mid m,s^2),
\qquad
p_j(z)=\sum_{k=1}^{K}\pi_k\mathcal{N}(z\mid \mu_k,v_k).
\end{equation}
Let
\begin{equation}
z=m+s\epsilon,
\qquad
\epsilon\sim\mathcal{N}(0,1),
\end{equation}
and define the component responsibilities
\begin{equation}
r_k(z)
=
\frac{
\pi_k\mathcal{N}(z\mid \mu_k,v_k)
}{
\sum_{\ell=1}^{K}\pi_\ell\mathcal{N}(z\mid \mu_\ell,v_\ell)
}.
\end{equation}
Then
\begin{equation}
\frac{\partial}{\partial m}\mathrm{KL}(q\|p_j)
=
\mathbb E_{\epsilon}
\left[
\sum_{k=1}^{K}
r_k(z)\frac{z-\mu_k}{v_k}
\right].
\end{equation}
If one component $k^\star$ dominates locally, so that $r_{k^\star}(z)\approx 1$ on the main support of $q$, then
\begin{equation}
\frac{\partial}{\partial m}\mathrm{KL}(q\|p_j)
\approx
\frac{m-\mu_{k^\star}}{v_{k^\star}},
\end{equation}
which shows that gradient descent on the KL term pushes the posterior mean toward the dominant component center.
\end{proposition}

\noindent\textit{Proof.}
The entropy of a Gaussian depends on $s^2$ but not on $m$, so
\begin{equation}
\frac{\partial}{\partial m}\mathbb E_q[\log q(z)]=0.
\end{equation}
Using the reparameterization $z=m+s\epsilon$ and the pathwise derivative estimator,
\begin{equation}
\frac{\partial}{\partial m}\mathrm{KL}(q\|p_j)
=
-
\frac{\partial}{\partial m}
\mathbb E_{\epsilon}\!\left[\log p_j(m+s\epsilon)\right]
=
-
\mathbb E_{\epsilon}\!\left[
\frac{\partial}{\partial z}\log p_j(z)
\right].
\end{equation}
Now
\begin{align}
\frac{\partial}{\partial z}\log p_j(z)
&=
\frac{
\sum_{k=1}^{K}
\pi_k \mathcal{N}(z\mid\mu_k,v_k)
\left(
-\frac{z-\mu_k}{v_k}
\right)
}{
\sum_{\ell=1}^{K}
\pi_\ell \mathcal{N}(z\mid\mu_\ell,v_\ell)
} \\
&=
-
\sum_{k=1}^{K}
r_k(z)\frac{z-\mu_k}{v_k}.
\end{align}
Substituting this identity yields
\begin{equation}
\frac{\partial}{\partial m}\mathrm{KL}(q\|p_j)
=
\mathbb E_{\epsilon}
\left[
\sum_{k=1}^{K}
r_k(z)\frac{z-\mu_k}{v_k}
\right].
\end{equation}
Under local dominance of component $k^\star$, the responsibilities satisfy $r_{k^\star}(z)\approx 1$ and $r_k(z)\approx 0$ for $k\neq k^\star$, so
\begin{equation}
\frac{\partial}{\partial m}\mathrm{KL}(q\|p_j)
\approx
\mathbb E_{\epsilon}\!\left[\frac{z-\mu_{k^\star}}{v_{k^\star}}\right]
=
\frac{m-\mu_{k^\star}}{v_{k^\star}}.
\end{equation}
This proves the claim. \hfill$\square$

Proposition~\ref{prop:posterior_mean_attraction} makes the specialization mechanism explicit. Once a latent dimension enters the attraction region of a particular mixture component, the KL gradient tends to keep pulling that posterior mean toward the corresponding component center. Since different latent dimensions possess different GMM parameters, they are subject to different attraction fields. This is a mathematical explanation for the source-wise specialization observed in practice.

The posterior variance in the present model requires separate discussion. Unlike the posterior means $\mu_{t,j}$, which are sample-dependent, the variance $\sigma_j^2$ is shared across all samples for a fixed latent dimension. Its dynamics therefore reflect an aggregate effect across samples rather than the behavior of a single local mode.

\begin{lemma}[Exact gradient of the one-dimensional KL term with respect to the shared posterior variance]
\label{lem:shared_variance_gradient}
Fix one latent dimension $j$, and denote
\begin{equation}
\lambda_j:=\sigma_j^2.
\end{equation}
For sample $t$, write
\begin{equation}
q_{t,j}(z)=\mathcal{N}(z\mid \mu_{t,j},\lambda_j),
\qquad
p_j(z)=\sum_{k=1}^{K_j}\pi_{j,k}\mathcal{N}(z\mid m_{j,k},v_{j,k}).
\end{equation}
Define the responsibilities
\begin{equation}
r_{tjk}(z)
=
\frac{
\pi_{j,k}\mathcal{N}(z\mid m_{j,k},v_{j,k})
}{
\sum_{\ell=1}^{K_j}
\pi_{j,\ell}\mathcal{N}(z\mid m_{j,\ell},v_{j,\ell})
},
\end{equation}
and the component scores
\begin{equation}
a_{jk}(z):=\frac{z-m_{j,k}}{v_{j,k}}.
\end{equation}
Then
\begin{equation}
\frac{\partial}{\partial \lambda_j}
\mathrm{KL}(q_{t,j}\|p_j)
=
\frac12
\left[
\mathbb E_{q_{t,j}}
\left(
\sum_{k=1}^{K_j}\frac{r_{tjk}(z)}{v_{j,k}}
\right)
-
\frac{1}{\lambda_j}
-
\mathbb E_{q_{t,j}}
\left(
\mathrm{Var}_{k\sim r_{tj}(z)}[a_{jk}(z)]
\right)
\right].
\end{equation}
Consequently, if one component $k_t^\star$ dominates locally for sample $t$, so that
\begin{equation}
r_{tjk_t^\star}(z)\approx 1
\quad\text{and}\quad
\mathrm{Var}_{k\sim r_{tj}(z)}[a_{jk}(z)]\approx 0,
\end{equation}
then
\begin{equation}
\frac{\partial}{\partial \lambda_j}
\mathrm{KL}(q_{t,j}\|p_j)
\approx
\frac12
\left(
\frac{1}{v_{j,k_t^\star}}
-
\frac{1}{\lambda_j}
\right).
\end{equation}
\end{lemma}

\noindent\textit{Proof.}
We use the Gaussian smoothing identity
\begin{equation}
\frac{\partial}{\partial \lambda_j}
\mathbb E_{q_{t,j}}[f(z)]
=
\frac12
\mathbb E_{q_{t,j}}[f''(z)],
\end{equation}
valid for sufficiently smooth $f$ under a Gaussian with variance $\lambda_j$.

First, since
\begin{equation}
\mathbb E_{q_{t,j}}[\log q_{t,j}(z)]
=
-\frac12\log(2\pi e\lambda_j),
\end{equation}
we have
\begin{equation}
\frac{\partial}{\partial \lambda_j}
\mathbb E_{q_{t,j}}[\log q_{t,j}(z)]
=
-\frac{1}{2\lambda_j}.
\end{equation}

Next, for the mixture prior,
\begin{equation}
\log p_j(z)
=
\log\sum_{k=1}^{K_j}\pi_{j,k}\mathcal{N}(z\mid m_{j,k},v_{j,k}),
\end{equation}
and differentiation gives
\begin{equation}
\frac{\partial}{\partial z}\log p_j(z)
=
-
\sum_{k=1}^{K_j} r_{tjk}(z)\,a_{jk}(z).
\end{equation}
Differentiating once more yields
\begin{equation}
\frac{\partial^2}{\partial z^2}\log p_j(z)
=
-
\sum_{k=1}^{K_j}\frac{r_{tjk}(z)}{v_{j,k}}
+
\mathrm{Var}_{k\sim r_{tj}(z)}[a_{jk}(z)].
\end{equation}
Therefore,
\begin{align}
\frac{\partial}{\partial \lambda_j}
\Bigl(
-\mathbb E_{q_{t,j}}[\log p_j(z)]
\Bigr)
&=
-\frac12
\mathbb E_{q_{t,j}}
\left[
\frac{\partial^2}{\partial z^2}\log p_j(z)
\right] \\
&=
\frac12
\left[
\mathbb E_{q_{t,j}}
\left(
\sum_{k=1}^{K_j}\frac{r_{tjk}(z)}{v_{j,k}}
\right)
-
\mathbb E_{q_{t,j}}
\left(
\mathrm{Var}_{k\sim r_{tj}(z)}[a_{jk}(z)]
\right)
\right].
\end{align}
Combining this with the derivative of $\mathbb E_{q_{t,j}}[\log q_{t,j}(z)]$ gives the claimed identity.

Under local dominance of one component $k_t^\star$, one has
\begin{equation}
\sum_{k=1}^{K_j}\frac{r_{tjk}(z)}{v_{j,k}}
\approx
\frac{1}{v_{j,k_t^\star}},
\qquad
\mathrm{Var}_{k\sim r_{tj}(z)}[a_{jk}(z)]\approx 0,
\end{equation}
so the approximation follows immediately. \hfill$\square$

Lemma~\ref{lem:shared_variance_gradient} clarifies an important distinction between posterior mean and posterior variance. The mean attraction in Proposition~\ref{prop:posterior_mean_attraction} is a sample-level local effect, whereas the shared variance $\sigma_j^2$ is optimized globally across all samples at latent dimension $j$.

Summing the approximate dominant-component expression in Lemma~\ref{lem:shared_variance_gradient} over $t=1,\dots,T$ therefore yields
\begin{equation}
\frac{\partial}{\partial \lambda_j}
\sum_{t=1}^{T}\mathrm{KL}(q_{t,j}\|p_j)
\approx
\frac{T}{2}
\left(
\overline{v_j^{-1}}
-
\frac{1}{\lambda_j}
\right),
\qquad
\overline{v_j^{-1}}
:=
\frac{1}{T}\sum_{t=1}^{T}\frac{1}{v_{j,k_t^\star}}.
\end{equation}
Hence the KL-only local stationary point satisfies
\begin{equation}
\lambda_j^\star
\approx
\left(
\overline{v_j^{-1}}
\right)^{-1}.
\end{equation}
That is, the shared posterior variance is generally driven toward a harmonic-type aggregate of the currently dominant mixture-component variances across samples, rather than toward any single component variance. Only in the special case where one component dominates for almost all samples does this reduce to
\begin{equation}
\lambda_j^\star \approx v_{j,k^\star}.
\end{equation}

The preceding observation extends naturally to the full training objective.

\begin{corollary}[Local variance balance in the full objective]
\label{cor:full_objective_variance_balance}
Write the expected training objective as
\begin{equation}
\mathcal J
=
\mathcal R(\phi,\theta,\boldsymbol{\sigma}^2)
+
\frac{\beta}{Tm}
\sum_{j=1}^{n}\sum_{t=1}^{T}\mathrm{KL}(q_{t,j}\|p_j),
\end{equation}
where $\mathcal R$ denotes the expected reconstruction term. Then
\begin{equation}
\frac{\partial \mathcal J}{\partial \lambda_j}
=
\frac{\partial \mathcal R}{\partial \lambda_j}
+
\frac{\beta}{Tm}
\sum_{t=1}^{T}
\frac{\partial}{\partial \lambda_j}\mathrm{KL}(q_{t,j}\|p_j).
\end{equation}
Under the dominant-component approximation,
\begin{equation}
\frac{\partial \mathcal J}{\partial \lambda_j}
\approx
\frac{\partial \mathcal R}{\partial \lambda_j}
+
\frac{\beta}{2m}
\left(
\overline{v_j^{-1}}
-
\frac{1}{\lambda_j}
\right).
\end{equation}
Thus the local stationary condition becomes
\begin{equation}
\frac{1}{\lambda_j^\star}
\approx
\overline{v_j^{-1}}
+
\frac{2m}{\beta}
\frac{\partial \mathcal R}{\partial \lambda_j}.
\end{equation}
In particular, when the local reconstruction contribution is nonnegative in the variance direction, the reconstruction term further favors a posterior variance smaller than or comparable to the KL-only harmonic-average scale.
\end{corollary}

The scope of the above results should be stated carefully. These statements analyze the local geometry of the KL-to-prior term
\begin{equation}
\mathrm{KL}\!\left(
q_\phi(z_{t,j}\mid \mathbf y_t)
\,\middle\|\,
p_{\psi_j}(z_{t,j})
\right),
\end{equation}
rather than the full exact posterior gap
\begin{equation}
\mathrm{KL}\!\left(
q_\phi(\mathbf z_t\mid \mathbf y_t)
\,\middle\|\,
p_{\theta,\psi}(\mathbf z_t\mid \mathbf y_t)
\right).
\end{equation}
Accordingly, the present analysis should be understood as a mechanism-level explanation of source-wise specialization induced by the adaptive per-dimension GMM regularizer, not as a full nonlinear identifiability theorem.

\subsection{Relation to the standard VAE}

The proposed model contains the standard VAE as a special case.

\begin{proposition}[Degeneration to the standard VAE]
Suppose that for every latent dimension $j$ we set $K=1$ and
\begin{equation}
\pi_{j,1}=1,
\qquad
\mu_{j,1}^{(p)}=0,
\qquad
(\sigma_{j,1}^{(p)})^2=1.
\end{equation}
Then the per-dimension GMM prior degenerates to
\begin{equation}
p_{\psi}(\mathbf{z}_t)=\mathcal{N}(\mathbf{z}_t\mid \mathbf{0},\mathbf{I}).
\end{equation}
Furthermore, under $\beta=1$, the implemented objective reduces to a normalized single-sample version of the standard VAE training objective. If the source-wise shared posterior variances are further replaced by sample-wise amortized posterior variances, the model becomes the usual amortized VAE.
\end{proposition}

\noindent\textit{Proof.}
If $K=1$ for every latent dimension and
\begin{equation}
\pi_{j,1}=1,\qquad
\mu_{j,1}^{(p)}=0,\qquad
(\sigma_{j,1}^{(p)})^2=1,
\end{equation}
then each prior factor becomes
\begin{equation}
p_{\psi_j}(z_{t,j})
=
\mathcal{N}(z_{t,j}\mid 0,1).
\end{equation}
Hence
\begin{equation}
p_{\psi}(\mathbf{z}_t)
=
\prod_{j=1}^{n}\mathcal{N}(z_{t,j}\mid 0,1)
=
\mathcal{N}(\mathbf{z}_t\mid \mathbf{0},\mathbf{I}).
\end{equation}
Under $\beta=1$, the implemented objective corresponds to a normalized Monte Carlo form of the negative ELBO used in the standard VAE.
\hfill$\square$

Thus, PDGMM-VAE should be understood not as a disconnected alternative to the VAE, but as a structured extension in which the latent prior is generalized from a shared simple prior to adaptive per-dimension mixture priors.

\subsection{Weak recovery in the linear low-noise limit}

The proposed framework does not by itself furnish a full unconditional nonlinear identifiability theorem. Nevertheless, in an idealized linear regime one can still state a weak recovery result showing how heterogeneous prior structure can reduce the residual equivalence class.

\begin{proposition}[Weak recovery in the linear low-noise limit]
Consider the square linear regime with $m=n$,
\begin{equation}
\mathbf{y}_t = \mathbf{A}\mathbf{s}_t + \boldsymbol{\xi}_t,
\qquad
\mathbf{A}\in\mathbb{R}^{n\times n}\ \text{invertible}.
\end{equation}
where the source components in $\mathbf{s}_t$ are mutually independent, the noise level is small, the decoder family contains all linear maps, the encoder mean family contains all linear maps, and the posterior variances satisfy $\sigma_j^2\to 0$. Assume further that each true source marginal is exactly representable by the designated per-dimension GMM prior family and that the one-dimensional source marginals are pairwise distinct up to possible sign symmetry. Then any optimum that achieves vanishing reconstruction error and exact marginal matching recovers the latent components up to the residual linear ICA equivalence class, and the heterogeneous per-dimension priors eliminate all permutations not preserving identical one-dimensional source marginals. In particular, when the source marginals are pairwise distinct up to sign, the residual permutation ambiguity is removed.
\end{proposition}

\noindent\textit{Proof sketch.}
In the linear low-noise regime, assume an optimum achieves vanishing reconstruction error and posterior variances tending to zero. Then the encoder-decoder composition behaves effectively as an invertible linear transform between the latent variables and the true sources. Classical linear ICA considerations imply that independent non-Gaussian components can only be recovered up to the standard residual equivalence class, consisting essentially of scaling, sign, and permutation ambiguities.

Now impose exact marginal matching with the designated per-dimension GMM priors. Because the prior attached to latent dimension $j$ is intended to represent a specific one-dimensional source marginal, any permutation that maps one latent dimension onto another with a different target marginal is incompatible with exact prior matching. Therefore only permutations preserving identical one-dimensional marginals remain feasible. If the source marginals are pairwise distinct up to sign, no nontrivial permutation remains. This yields the claimed weak ambiguity reduction. \hfill$\square$

Proposition 4 is intentionally weak and local. Its role is not to claim full nonlinear identifiability, but to clarify that in ideal linear low-noise settings the heterogeneous prior can further reduce the ambiguity that remains after classical linear ICA considerations.

\label{subsec:algorithm}

Algorithm~\ref{alg:gmmvae_ica} summarizes the overall training procedure of the proposed PDGMM-VAE.

\begin{algorithm}[htbp]
\footnotesize
\setlength{\algomargin}{0.8em}
\SetAlgoLined
\DontPrintSemicolon
\SetNlSkip{0.2em}
\SetInd{0.2em}{0.4em}
\caption{Training procedure of the proposed PDGMM-VAE}
\label{alg:gmmvae_ica}

\KwIn{Observed mixtures $\mathbf{Y}=\{\mathbf{y}_t\}_{t=1}^{T}$, observation dimension $m$, latent dimension $n$, GMM component number $K$, encoder parameters $\phi$, decoder parameters $\theta$, shared posterior variances $\boldsymbol{\sigma}^2$, prior parameters $\psi=\{\alpha_{j,k},\mu_{j,k}^{(p)},\eta_{j,k}\}$, and weighting coefficient $\beta$.}

\KwOut{Optimized parameters $(\phi,\theta,\boldsymbol{\sigma}^2,\psi)$ and variational posterior $q_{\phi}(\mathbf{Z}\mid\mathbf{Y})$.}

Initialize $(\phi,\theta,\boldsymbol{\sigma}^2,\psi)$\;

\While{not converged}{

Compute $\boldsymbol{\mu}_t=f_\phi(\mathbf{y}_t)$ for $t=1,\dots,T$\;

Compute
$\pi_{j,k}=\dfrac{\exp(\alpha_{j,k})}{\sum_{\ell=1}^{K}\exp(\alpha_{j,\ell})}$
and
$(\sigma_{j,k}^{(p)})^2=\exp(\eta_{j,k})$
for $j=1,\dots,n$, $k=1,\dots,K$\;

Sample one reparameterized latent realization
$\tilde z_{t,j}=\mu_{t,j}+\sigma_j\epsilon_{t,j}$
with $\epsilon_{t,j}\sim\mathcal{N}(0,1)$
for $t=1,\dots,T$, $j=1,\dots,n$\;

Reconstruct $\hat{\mathbf{y}}_t=g_\theta(\tilde{\mathbf z}_t)$ for $t=1,\dots,T$\;

Compute
$\mathcal{L}_{\mathrm{rec}}=\dfrac{1}{Tm}\sum_{t=1}^{T}\|\hat{\mathbf y}_t-\mathbf y_t\|_2^2$\;

Compute
$\log q_{\phi}(\widetilde{\mathbf Z}\mid\mathbf Y)
=
\sum_{t=1}^{T}\sum_{j=1}^{n}
\log\mathcal{N}(\tilde z_{t,j}\mid\mu_{t,j},\sigma_j^2)$\;

\For{$j\leftarrow 1$ \KwTo $n$}{
Compute
$\log p_{\psi_j}(\tilde z_{t,j})
=
\log\!\left[
\sum_{k=1}^{K}
\pi_{j,k}\,
\mathcal{N}\!\left(
\tilde z_{t,j}\mid \mu_{j,k}^{(p)},(\sigma_{j,k}^{(p)})^2
\right)
\right]$
for $t=1,\dots,T$\;
}

Compute
$\log p_{\psi}(\widetilde{\mathbf Z})
=
\sum_{t=1}^{T}\sum_{j=1}^{n}\log p_{\psi_j}(\tilde z_{t,j})$\;

Form
$\mathcal{L}
=
\mathcal{L}_{\mathrm{rec}}
+
\frac{\beta}{Tm}
\left[
\log q_{\phi}(\widetilde{\mathbf Z}\mid\mathbf Y)
-
\log p_{\psi}(\widetilde{\mathbf Z})
\right]$\;

Update $(\phi,\theta,\boldsymbol{\sigma}^2,\psi)$ by backpropagation and an optimizer step\;
}

\Return $(\phi,\theta,\boldsymbol{\sigma}^2,\psi)$ and
$q_{\phi}(\mathbf{Z}\mid\mathbf{Y})
=
\prod_{t=1}^{T}\prod_{j=1}^{n}
\mathcal{N}(z_{t,j}\mid\mu_{t,j},\sigma_j^2)$\;
\end{algorithm}

\subsection{Interpretation of the proposed design}

As illustrated in Figure~1, the proposed PDGMM-VAE can be interpreted as a source-oriented variational autoencoder for nonlinear ICA. The encoder maps the observed mixtures to latent source-like representations, where each latent dimension is treated as an individual source component and regularized by its own adaptive Gaussian mixture prior. The decoder then maps the sampled latent variables back to the observation space and reconstructs the mixtures.

The theoretical analysis above clarifies that the role of the adaptive per-dimension GMM prior is not merely to provide a flexible family of marginal densities. First, the exact factorization and ELBO derivation show that the model remains a well-defined structured-prior variational latent-variable model rather than an ad hoc reconstruction-plus-regularization scheme. Second, the comparison between homogeneous and heterogeneous priors shows that the proposed per-dimension design weakens much of the latent permutation symmetry present under shared priors. Third, the KL analysis shows that the prior term induces component-wise attraction fields, thereby encouraging different latent dimensions to move toward different source-specific marginal structures.

Accordingly, PDGMM-VAE should not be interpreted as already furnishing a full unconditional nonlinear identifiability theorem. Rather, its current theoretical position is that adaptive per-dimension GMM priors provide a principled structured mechanism that reduces symmetry-related ambiguity, induces source-wise specialization through the KL term, and, in idealized linear low-noise regimes, can further reduce the residual equivalence class of source recovery. This is precisely the sense in which adaptive per-dimension prior design plays a central role in guiding latent source recovery in the proposed framework.

\begin{figure}[htbp]
    \centering
    \includegraphics[width=\textwidth]{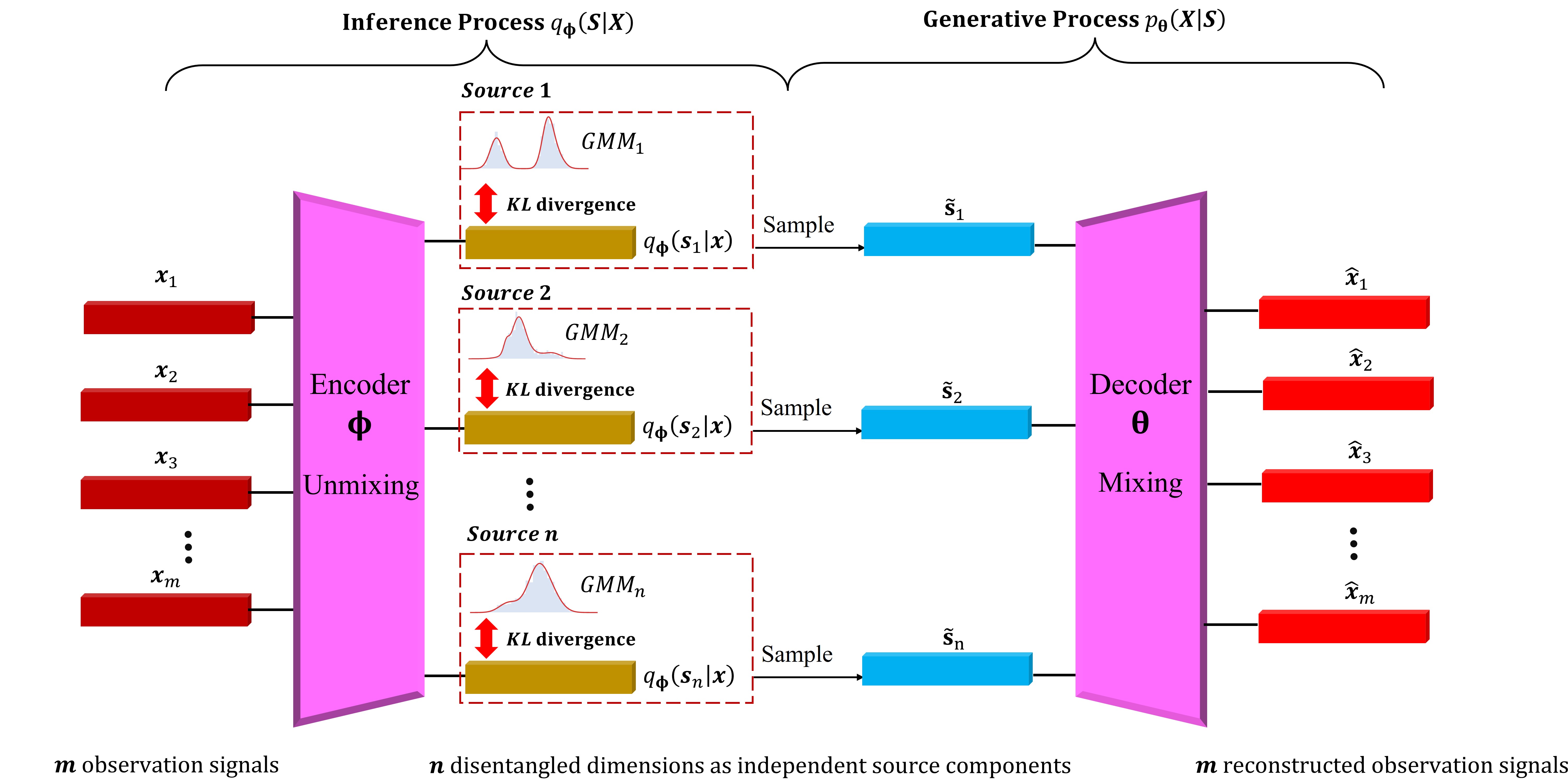}
    \caption{Illustration of the proposed PDGMM-VAE framework. Each latent dimension is regularized by an adaptive Gaussian mixture prior whose parameters are jointly learned during training.}
    \label{fig:pdgmmvae_framework}
\end{figure}

\section{Experimental Study}

\subsection{Linear ICA scenario}

We first evaluate PDGMM-VAE under a linear mixing setting. In this experiment, three i.i.d.\ latent sources with different non-Gaussian marginal distributions are mixed linearly and then recovered by the proposed model. Figure~2 shows the training dynamics. The total loss decreases rapidly and stabilizes, while the posterior variances and the GMM parameters gradually converge to stable values. This indicates that the per-dimension GMM priors are not fixed beforehand, but are adaptively optimized together with the encoder and decoder under the overall objective.

Figure~3 compares the true sources and the inferred posterior means after z-score normalization, together with the posterior uncertainty bands. The recovered source means closely follow the ground-truth sources across all three components, showing that the model captures both the main waveform and the uncertainty structure well. The final absolute correlations between the true sources and the inferred posterior means are
\[
|\mathrm{corr}_1| = 0.9988,\qquad
|\mathrm{corr}_2| = 0.9963,\qquad
|\mathrm{corr}_3| = 0.9907.
\]
These results indicate highly accurate source recovery in the linear case.

Figures~4--6 further compare the true source distributions and the estimated source distributions for the three source components. In each case, the estimated histogram and the learned GMM closely match the corresponding true distribution, showing that the adaptive per-dimension priors can capture source-specific non-Gaussian marginals effectively. It should be noted that the individual GMM component parameters do not need to match the true components one by one exactly. Different combinations of mixture weights, means, and variances may produce very similar overall mixture densities, so agreement at the level of the full distribution is more important than exact component-wise correspondence.

\begin{figure}[htbp]
    \centering
    \includegraphics[width=\textwidth]{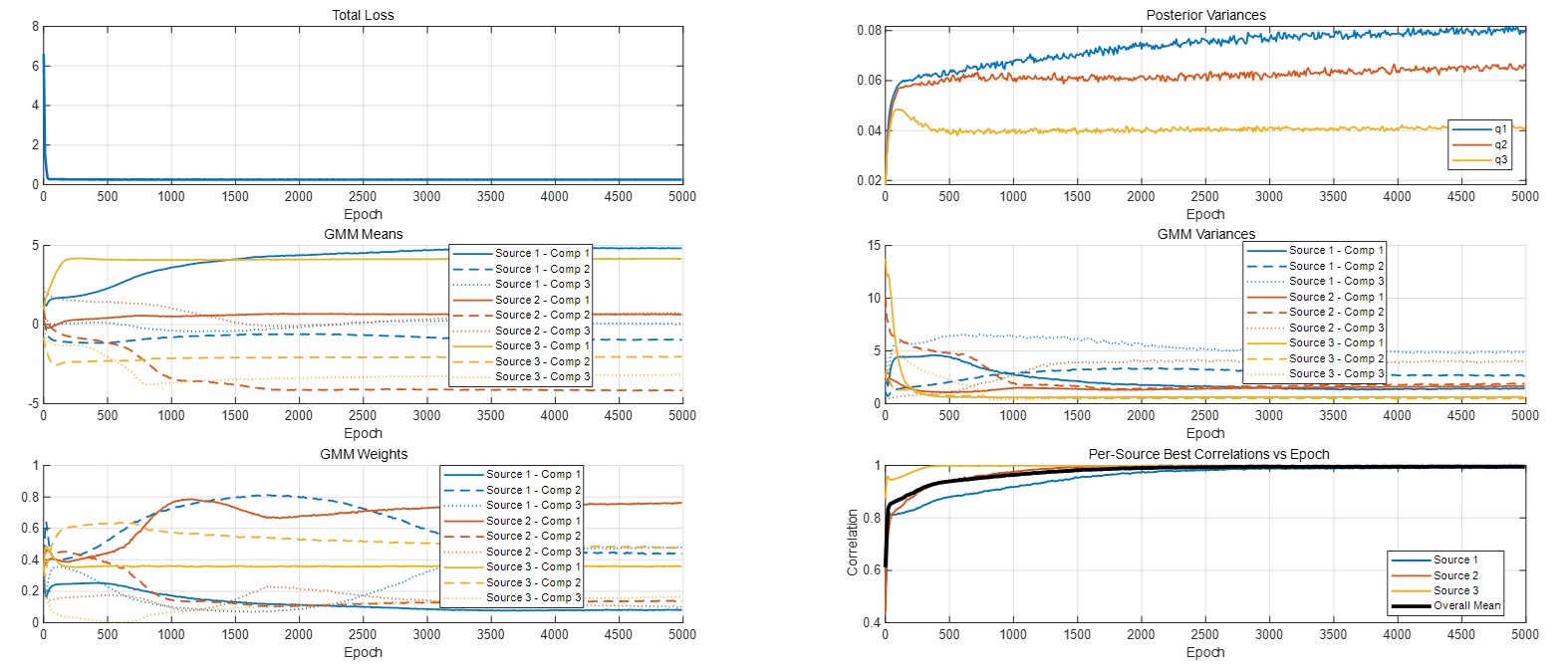}
    \caption{Training curves in the linear ICA experiment, including the total loss, posterior variances, GMM means, GMM variances, GMM weights, and per-source maximum correlations.}
    \label{fig:linear_training}
\end{figure}

\begin{figure}[htbp]
    \centering
    \includegraphics[width=\textwidth]{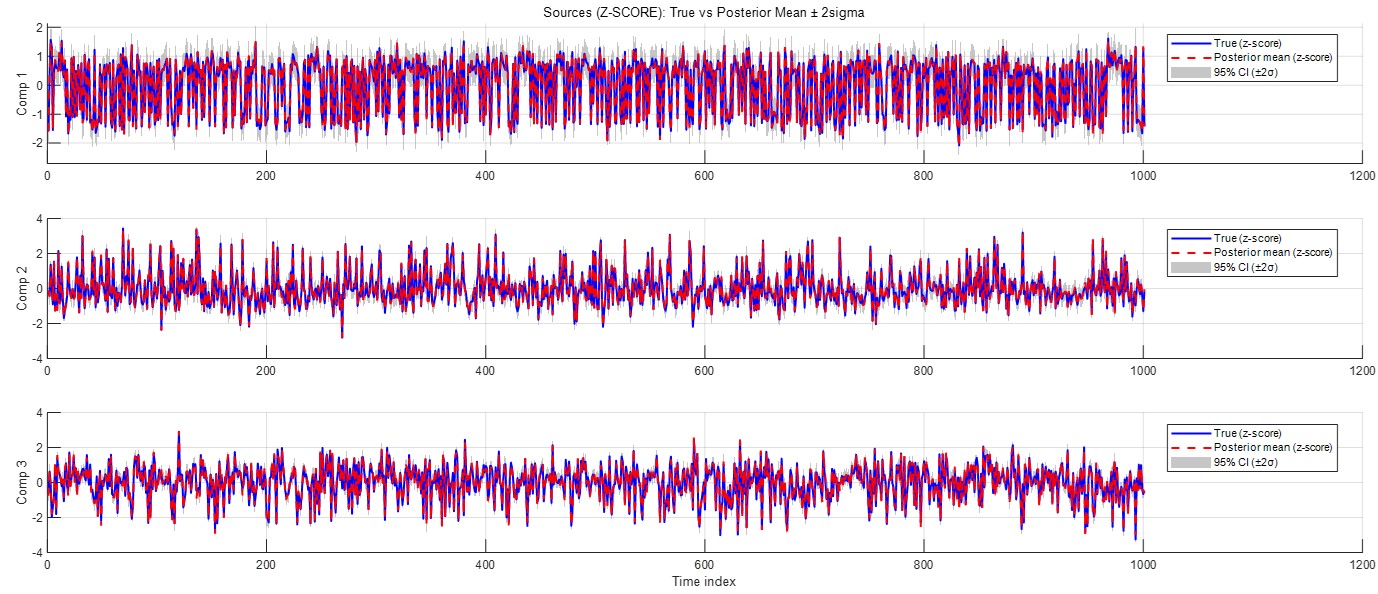}
    \caption{Comparison between the true sources and the inferred posterior means in the linear ICA experiment after z-score normalization. The shaded regions denote the posterior uncertainty bands.}
    \label{fig:linear_sources}
\end{figure}

\begin{figure}[htbp]
    \centering
    \includegraphics[width=\textwidth]{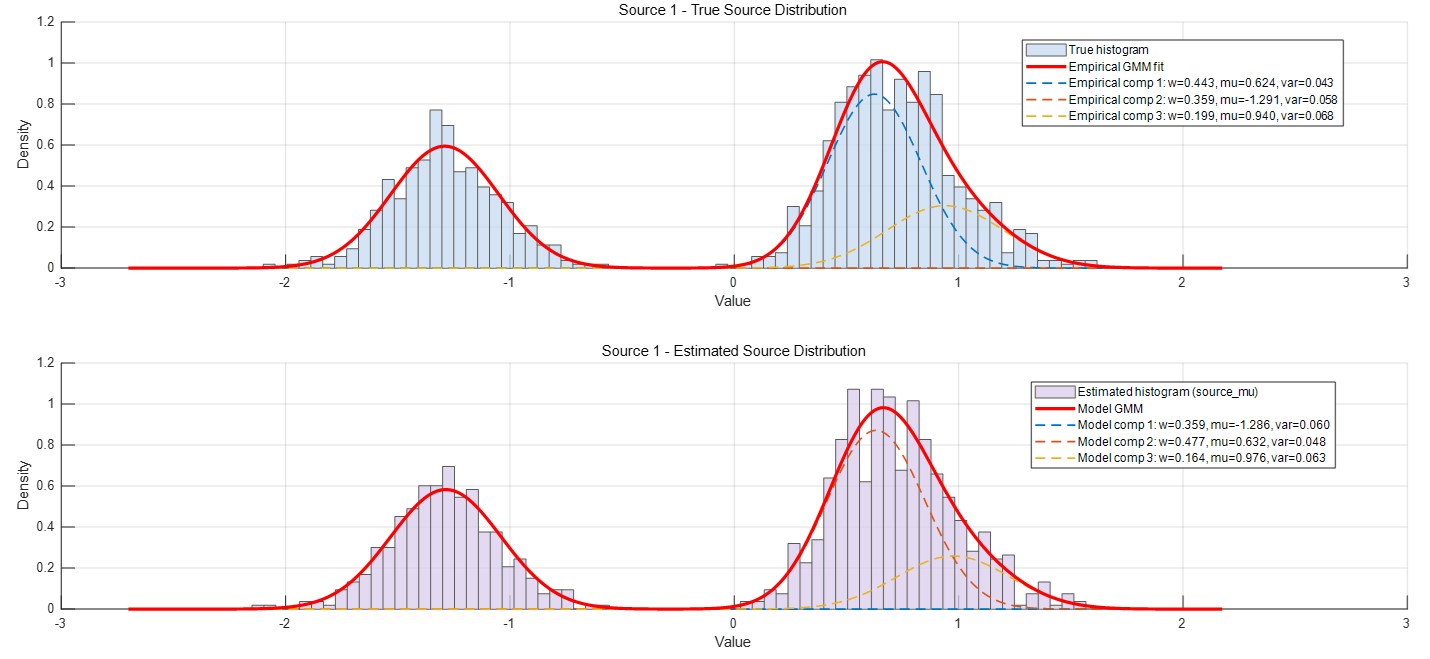}
    \caption{True and estimated source distributions for Source 1 in the linear ICA experiment.}
    \label{fig:linear_gmm_s1}
\end{figure}

\begin{figure}[htbp]
    \centering
    \includegraphics[width=\textwidth]{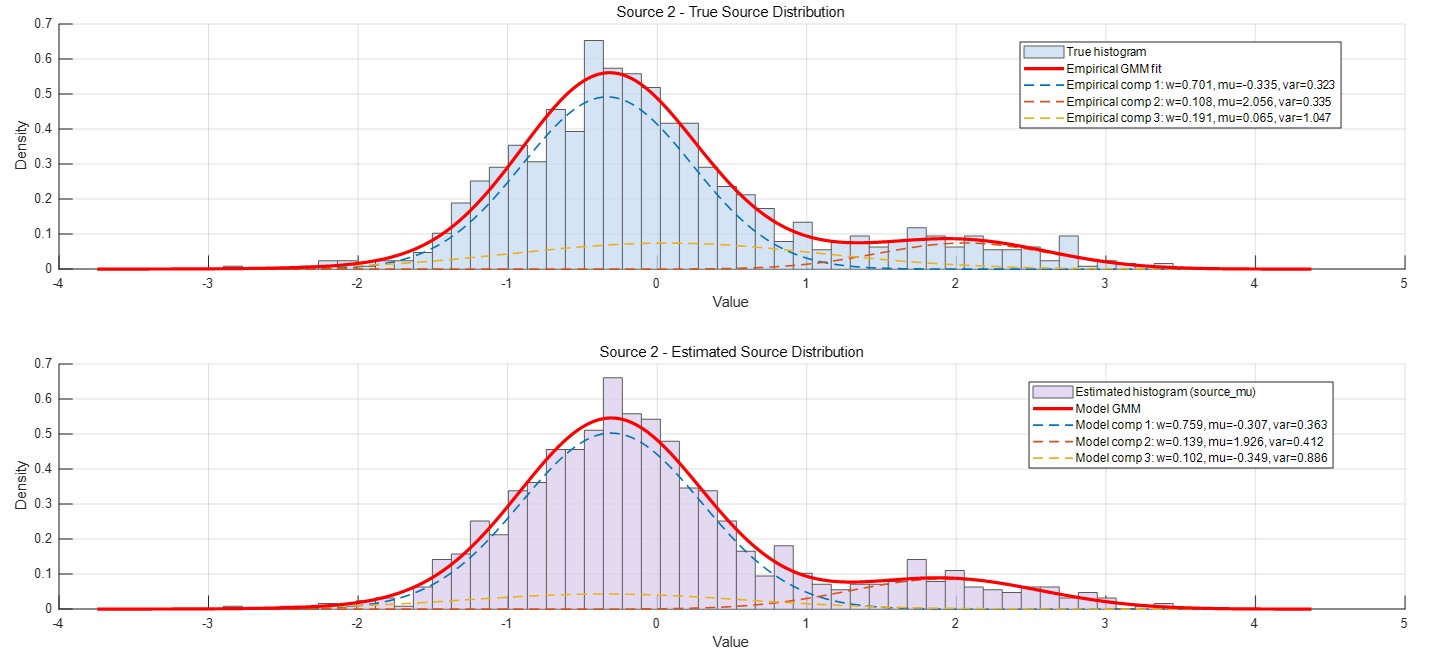}
    \caption{True and estimated source distributions for Source 2 in the linear ICA experiment. }
    \label{fig:linear_gmm_s2}
\end{figure}

\begin{figure}[htbp]
    \centering
    \includegraphics[width=\textwidth]{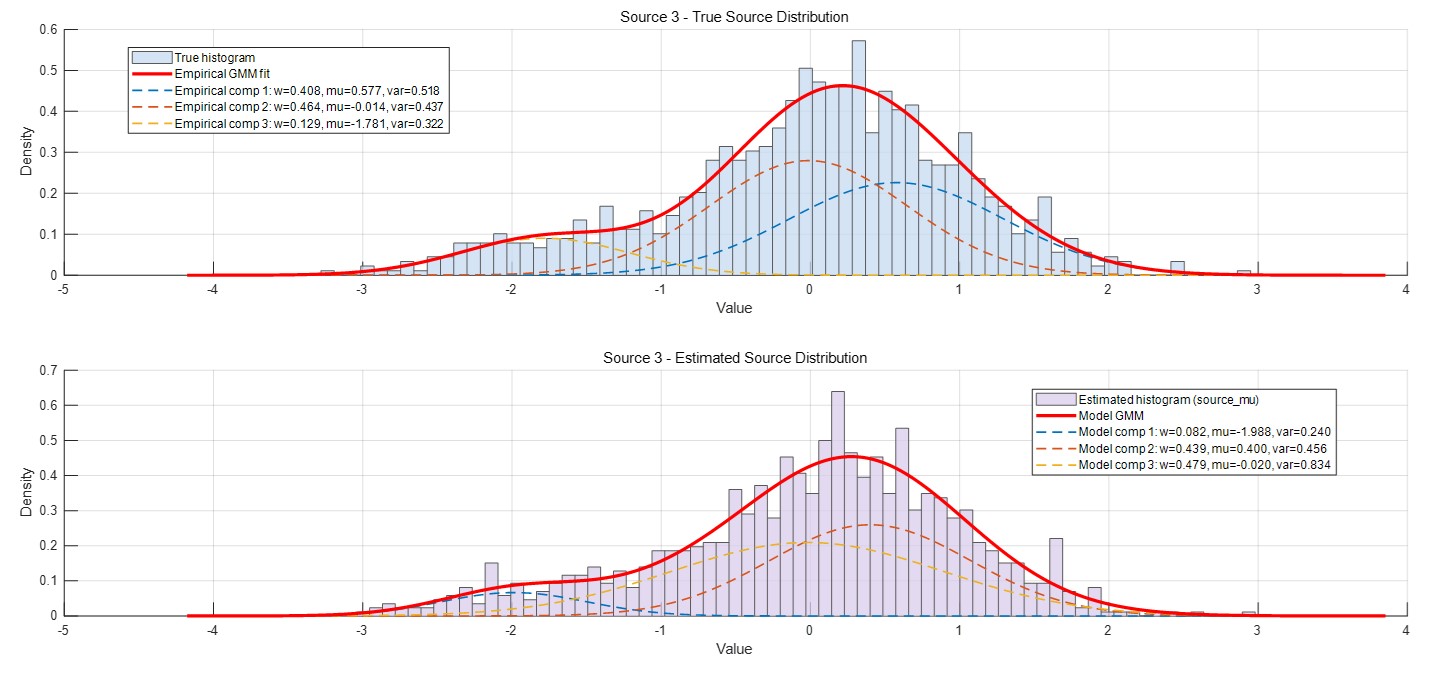}
    \caption{True and estimated source distributions for Source 3 in the linear ICA experiment.}
    \label{fig:linear_gmm_s3}
\end{figure}
\subsection{Nonlinear ICA scenario}

We next evaluate PDGMM-VAE under a nonlinear mixing setting. Starting from the latent source vector $\mathbf{z}_t$, the observations are generated through the following nonlinear transformation:
\begin{equation}
\mathbf{h}_t = \tanh(\mathbf{A}_1 \mathbf{z}_t), \qquad
\mathbf{x}_t = \tanh(\mathbf{A}_2 \mathbf{h}_t),
\end{equation}
where $\mathbf{A}_1$ and $\mathbf{A}_2$ are mixing matrices and $\tanh(\cdot)$ is applied element-wise. This construction introduces nonlinear distortions beyond ordinary linear mixing and therefore provides a more challenging source recovery problem.

Figure~7 shows the training dynamics in the nonlinear case. As in the linear experiment, the total loss decreases rapidly and the posterior variances and GMM parameters gradually stabilize, indicating that the adaptive per-dimension priors remain trainable and converge under the overall objective. Figure~8 compares the true sources and the inferred posterior means after z-score normalization. The recovered posterior means still follow the true sources closely, although the fitting quality is slightly weaker than in the linear case, which is expected under stronger nonlinear distortions. The final absolute correlations are
\[
|\mathrm{corr}_1| = 0.9943,\qquad
|\mathrm{corr}_2| = 0.9693,\qquad
|\mathrm{corr}_3| = 0.9593.
\]
These values show that the proposed model still achieves satisfactory source recovery under nonlinear mixing.

Figures~9--11 compare the true and estimated source distributions for the three latent sources. In all three cases, the learned per-dimension GMM priors remain relatively close to the target marginals and capture the main non-Gaussian structure of each source. As in the linear setting, exact one-to-one correspondence between individual mixture components is not required, since different combinations of component weights, means, and variances may produce very similar overall mixture densities. Overall, although the nonlinear case is clearly more difficult than the linear one, the proposed method still provides acceptable recovery accuracy and distributional fitting performance.

\begin{figure}[htbp]
    \centering
    \includegraphics[width=\textwidth]{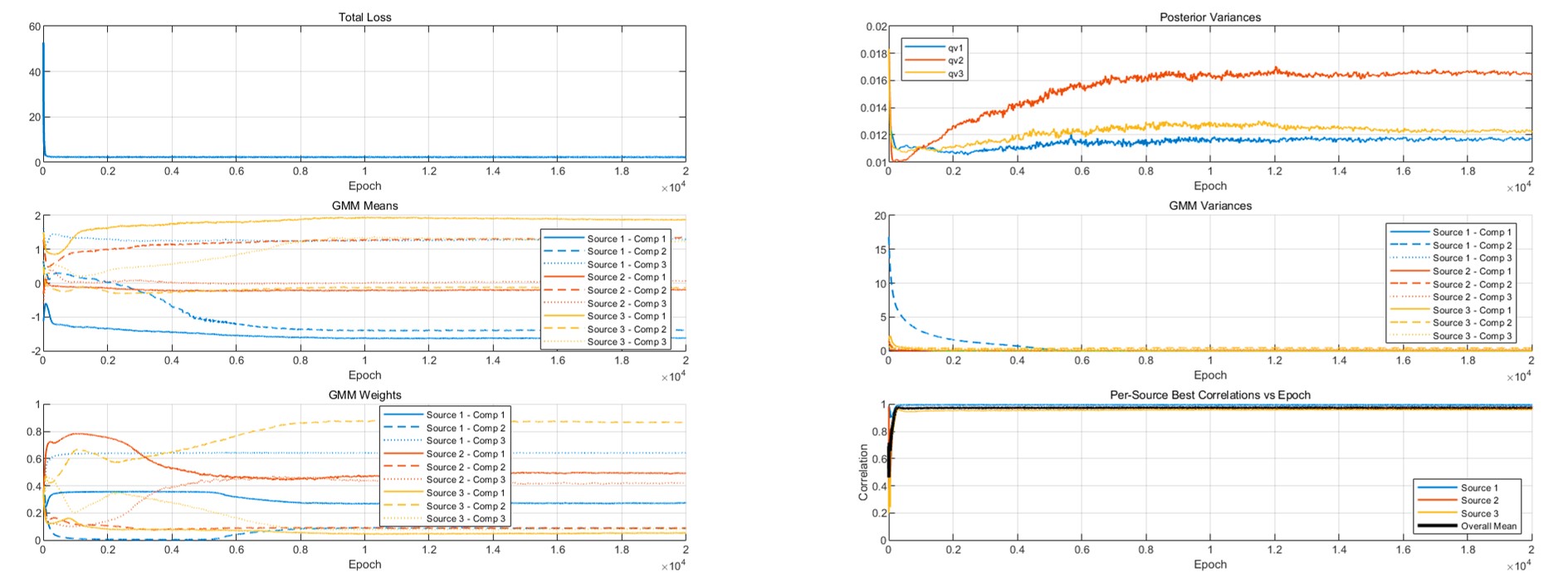}
    \caption{Training curves in the nonlinear ICA experiment, including the total loss, posterior variances, GMM means, GMM variances, GMM weights, and per-source maximum correlations.}
    \label{fig:nonlinear_training}
\end{figure}

\begin{figure}[htbp]
    \centering
    \includegraphics[width=\textwidth]{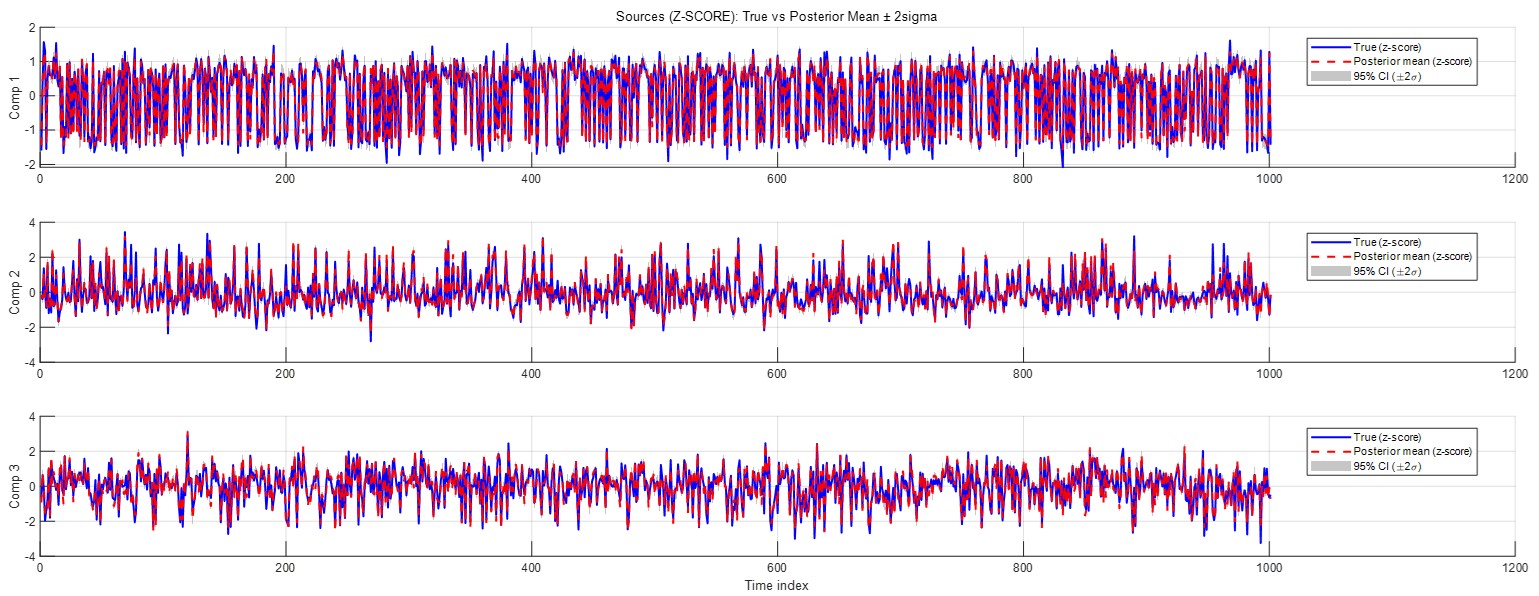}
    \caption{Comparison between the true sources and the inferred posterior means in the nonlinear ICA experiment after z-score normalization. The shaded regions denote the posterior uncertainty bands.}
    \label{fig:nonlinear_sources}
\end{figure}

\begin{figure}[htbp]
    \centering
    \includegraphics[width=\textwidth]{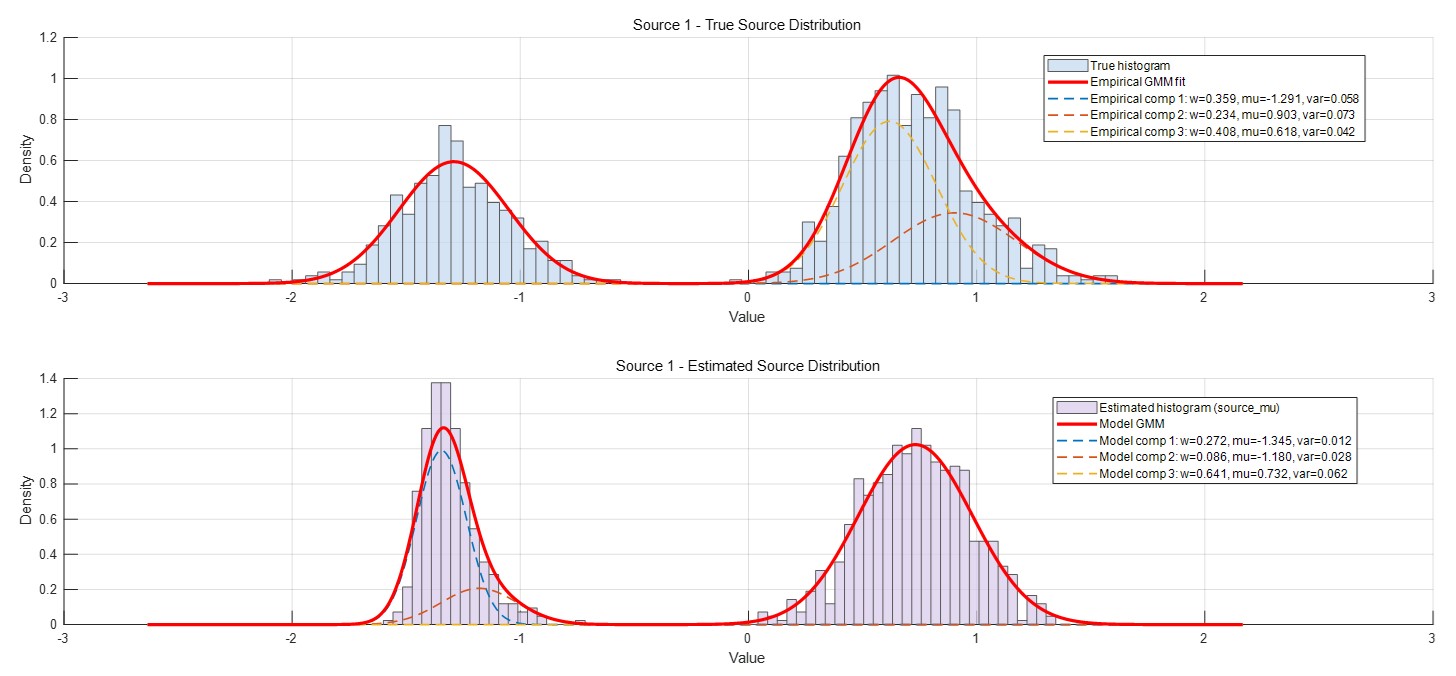}
    \caption{True and estimated source distributions for Source 1 in the nonlinear ICA experiment.}
    \label{fig:nonlinear_gmm_s1}
\end{figure}

\begin{figure}[htbp]
    \centering
    \includegraphics[width=\textwidth]{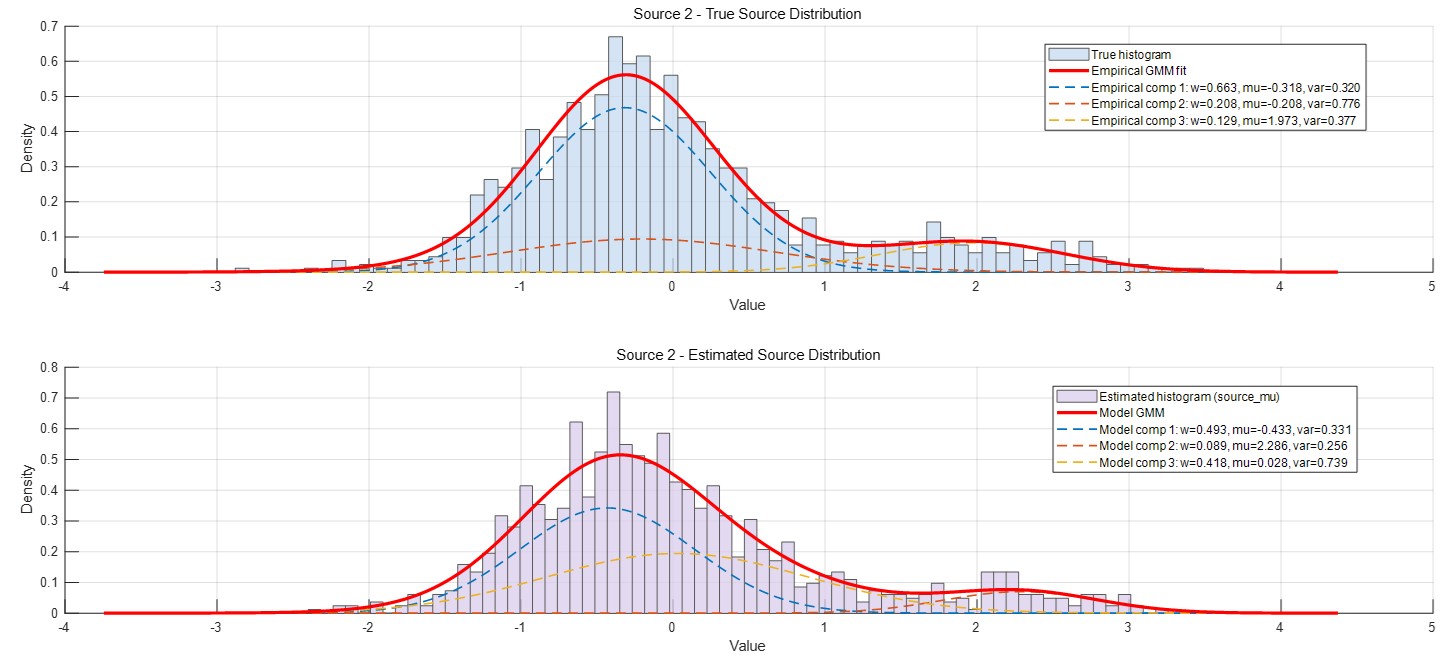}
    \caption{True and estimated source distributions for Source 2 in the nonlinear ICA experiment.}
    \label{fig:nonlinear_gmm_s2}
\end{figure}

\begin{figure}[htbp]
    \centering
    \includegraphics[width=\textwidth]{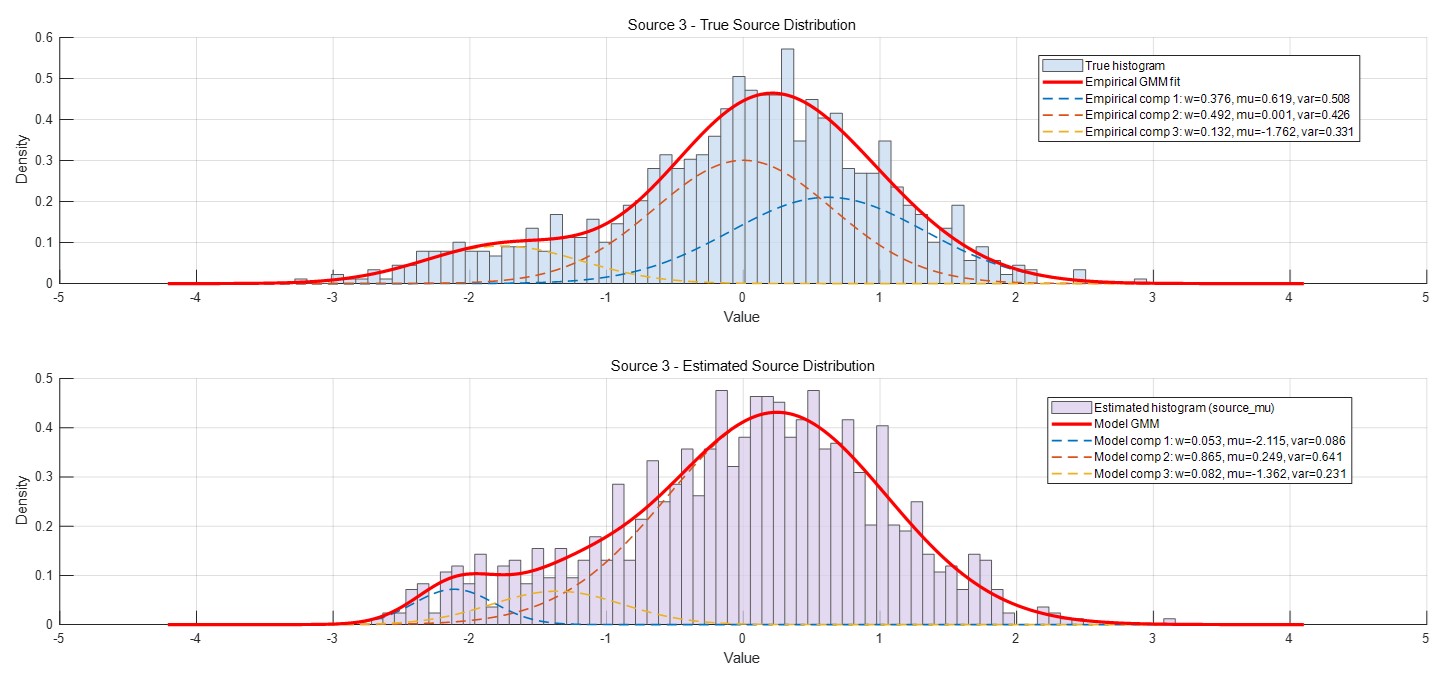}
    \caption{True and estimated source distributions for Source 3 in the nonlinear ICA experiment. }
    \label{fig:nonlinear_gmm_s3}
\end{figure}

\section{Conclusion and Future Work}

In this work, we proposed PDGMM-VAE, an adaptive per-dimension Gaussian-mixture-prior variational autoencoder for ICA-oriented source recovery. In the proposed framework, latent dimensions are interpreted explicitly as source components, and each dimension is assigned its own learnable GMM prior, whose parameters are jointly optimized together with the encoder and decoder under the overall training objective. In this way, the model introduces heterogeneous source-specific prior constraints into a unified probabilistic encoder-decoder framework and enables different latent dimensions to capture different non-Gaussian source marginals.

Beyond the empirical source recovery results obtained in both linear and nonlinear mixing settings, the present paper also provides a clearer theoretical account of why the proposed prior design is meaningful. First, the model is formulated explicitly as a structured-prior variational latent-variable model, together with an ELBO interpretation and an implemented normalized training objective. Second, we show that homogeneous shared priors preserve substantial latent permutation symmetry, whereas heterogeneous per-dimension priors reduce this symmetry and thereby reduce the set of equivalent latent relabelings. Third, we show that the KL regularization induced by the adaptive per-dimension GMM priors admits a component-attraction interpretation, which helps explain why different latent dimensions tend to specialize toward different source roles during training. We also clarify that PDGMM-VAE contains the standard VAE as a special case, and we provide a weak recovery statement in an idealized linear low-noise regime.

At the same time, the current theoretical conclusions should be interpreted with appropriate caution. The present analysis does not establish a full unconditional nonlinear identifiability theorem. Rather, it shows that adaptive per-dimension GMM priors provide a principled mechanism for symmetry reduction, source-wise specialization, and weaker ambiguity reduction in analytically tractable regimes. A more complete nonlinear identifiability theory for adaptive structured-prior VAEs remains an open problem.

Several directions remain for future work. First, although the present analysis clarifies the roles of symmetry reduction and KL-induced specialization, stronger convergence guarantees and sharper optimization characterizations for jointly learned per-dimension mixture priors would still be valuable. Second, the current study focuses on i.i.d.\ source signals and emphasizes source-wise marginal non-Gaussianity rather than explicit temporal or spatial structure. It would therefore be interesting to combine the present per-dimension mixture design with temporal, spatial, or more general dependency-aware priors. Third, it would be valuable to investigate adaptive mixture-prior VAEs under richer theoretical settings, such as conditional structured priors, auxiliary-variable formulations, or other settings in which stronger identifiability results may become available. We hope that the present work can serve as a useful starting point for these developments.

\clearpage

\bibliography{ref}

\end{document}